\documentclass{article}

\usepackage{microtype}
\usepackage{graphicx, float}
\usepackage{subfigure}
\usepackage{booktabs} 
\usepackage{xcolor}

\newcommand{\blue}[1]{\textcolor{blue}{#1}}

\usepackage{hyperref}

\usepackage[accepted]{icml2025}

\usepackage{amsmath}
\usepackage{amssymb}
\usepackage{mathtools}
\usepackage{amsthm}
\usepackage{amsmath}

\DeclareMathOperator*{\argmin}{arg\,min}

\usepackage[capitalize,noabbrev]{cleveref}

\theoremstyle{plain}
\newtheorem{theorem}{Theorem}[section]

\newtheorem{lemma}[theorem]{Lemma}
\newtheorem{corollary}[theorem]{Corollary}
\theoremstyle{definition}
\newtheorem{definition}[theorem]{Definition}
\newtheorem{assumption}[theorem]{Assumption}
\theoremstyle{remark}
\newtheorem{remark}[theorem]{Remark}
\usepackage[textsize=tiny]{todonotes}

\icmltitlerunning{Online Algorithm for Personalised LLM Routing}
\begin{document}

\twocolumn[
\icmltitle{Chance-Constrained Personalized LLM Routing with Online Preference Learning}\icmlsetsymbol{equal}{*}

\begin{icmlauthorlist}
\icmlauthor{Ameya Anjarlekar}{yyy}
\icmlauthor{Rasoul Etesami}{yyy,comp}
\icmlauthor{R Srikant}{yyy}
\end{icmlauthorlist}

\icmlaffiliation{yyy}{Electrical and Compute Engineering, UIUC}
\icmlaffiliation{comp}{Department of Industrial and Enterprise Systems Engineering, UIUC}

\icmlcorrespondingauthor{Ameya Anjarlekar}{ameyasa2@illinois.edu}
\icmlcorrespondingauthor{Rasoul Etesami}{etesami1@illinois.edu}
\icmlcorrespondingauthor{R Srikant}{rsrikant@illinois.edu}

\icmlkeywords{LLM Routing, Online Learning}

\vskip 0.3in
]

\begin{abstract}
We study personalized LLM routing, where a router selects among
heterogeneous models with different costs and capabilities. The router
must minimize cumulative inference cost while ensuring reliable
realized response quality, even though user preferences are unknown and
must be learned from feedback during deployment. We formulate this
problem as a constrained contextual bandit and develop a primal-dual
GLM-UCB algorithm that learns preferences while balancing cost and
quality. Under a logistic preference model and a strict-feasibility
condition, we establish sublinear expected cost regret relative to an
optimal stationary randomized contextual oracle and a high-probability
guarantee on realized cumulative satisfaction. The framework extends
beyond LLM routing to broader online model-selection problems with
reliability requirements.
\end{abstract}

\section{Introduction}
\label{sec:introduction}

Modern machine learning systems increasingly rely on collections of
models with different capabilities and inference costs. Rather than
using the same model for every input, a system can route each request
to an appropriate model, using inexpensive models when they are
sufficient and reserving more expensive models for difficult inputs.
This model-routing paradigm is particularly relevant for large
language models (LLMs), whose costs and capabilities can differ
substantially across models
\cite{chen2024frugalgpt,hu2024routerbench,ong2025routellm}.

The challenge becomes more subtle when the appropriate model depends
not only on the request but also on the user. Different users may
prefer different models for the same request, and these preferences
may not be known before deployment. A router must therefore learn a
user's preferences from feedback while making routing decisions,
gradually discovering which model is more likely to satisfy that user
for a given request. This connection between personalization and
online learning has long been studied through contextual-bandit
methods for personalized recommendation
\cite{li2010contextual,xu2020timevarying}, and recent work has begun to
explore personalized and online LLM routing
\cite{dai2025personalizedrouter,xie2025gmtrouter,zeng2026metarouter}.
These developments motivate an online routing problem in which user
preferences are unknown and must be learned from the outcomes of the
models selected by the router.

Learning user preferences, however, is only part of the challenge. In
many deployments, consistently poor responses may be unacceptable even
if average satisfaction is high. A router may instead need to ensure
that at least a $(1-\epsilon)$ fraction of responses satisfy a desired
quality criterion over a horizon of $T$ requests, with probability at
least $1-\delta$. If $Y_t\in\{0,1\}$ indicates whether the response at
round $t$ satisfies the user's requirements, this requirement is
\begin{equation}
\mathbb{P}\left(
\sum_{t=1}^{T}Y_t \geq (1-\epsilon)T
\right) \ \geq \ 1-\delta
\label{eq:quality_constraint_intro}
\end{equation}
Here, the satisfaction outcome depends on both the selected model and the user's preferences. Therefore, the router must learn the preferences governing satisfaction likelihood while ensuring that the resulting sequence of outcomes remains sufficiently reliable by satisfying the chance constraint.

This leads to a challenging online decision problem. The router must
learn from the limited feedback generated by its own model choices,
keep cumulative inference cost low, and ensure that the realized
number of satisfactory responses does not fall too far below the
target. This differs from an expectation-level constraint, which can
balance expected cost and expected satisfaction through a Lagrangian
objective. Here, controlling expected satisfaction alone does not
directly control the probability of falling below the desired level in
the realized sequence. The router must therefore account for both
uncertainty in what it has learned about the user's preferences and
randomness in the resulting satisfaction outcomes.

We formulate this setting as a constrained contextual bandit with
known model costs and unknown preference parameters. At each round,
the learner observes the request context, selects one model, incurs its
cost, and observes only the satisfaction feedback from that model. The
goal is to minimize cumulative cost while satisfying the
high-probability requirement in
\eqref{eq:quality_constraint_intro}.

Our solution follows a simple principle: learn what the user prefers
while continuously monitoring whether the routing decisions are
keeping pace with the desired satisfaction level. The router maintains
optimistic estimates of how likely each model is to satisfy the user,
allowing it to explore uncertain choices while favoring cost-effective
ones. At the same time, it keeps track of any accumulated deficit in
satisfaction. When the router begins to fall behind, it shifts toward
models that are more likely to satisfy the user. We use this idea to propose an 
algorithm and provide sublinear regret bounds with respect to a stationary routing policy while ensuring that the chance constraints are met.

Although personalized LLM routing motivates our study, the framework applies more broadly. It applies whenever a system repeatedly chooses among heterogeneous models, services, or alternatives with different costs, while learning an unknown preference or satisfaction model from feedback and seeking to maintain a reliability requirement. This includes, for example, personalized recommendation and online advertising under resource or budget constraints \cite{budget}, cloud or edge service selection subject to cost, quality-of-service, and reliability requirements \cite{Ma2026EfficientSS}, and dynamic procurement among suppliers with heterogeneous costs and uncertain quality under operational constraints \cite{vassos2023}.

These considerations lead to the following contributions:
\begin{itemize}
\item We formulate personalized online LLM routing as a contextual bandit
with unknown user preferences and a cumulative requirement on
realized satisfaction, going beyond the usual expectation-level
cost-quality formulation.

\item We develop a primal-dual GLM-UCB algorithm that combines
optimistic preference learning with a virtual satisfaction-deficit
queue and a safety buffer.

\item Under a logistic preference model and a uniform strict-feasibility
condition, we establish sublinear expected cost regret relative to an
optimal stationary randomized contextual oracle while guaranteeing
the required realized cumulative satisfaction with high probability.

\item Finally, we provide numerical experiments to demonstrate the effectiveness of our algorithm. We evaluate its performance through synthetic experiments and a real-world LLM routing case study, examining chance constraint satisfaction, routing cost, personalization, and finite-horizon behavior.
\end{itemize}

\begin{table}[t]
\centering
\caption{Comparison of LLM routing approaches.}
\label{tab:routing_comparison}
\setlength{\tabcolsep}{4pt}
\renewcommand{\arraystretch}{1.05}
\begin{tabular}{lccc}
\toprule
Method & Online & Personalized & Safety \\
\midrule
FrugalGPT & $\times$ & $\times$ & $\times$ \\
RouteLLM & $\times$ & $\times$ & $\times$ \\
PersonalizedRouter & $\times$ & $\checkmark$ & $\times$ \\
MetaRouter & $\checkmark$ & $\checkmark$ & $\times$ \\
BaRP & $\checkmark$ & $\checkmark$ & $\times$ \\
Conformal Routing & $\times$ & $\times$ & $\checkmark$ \\
SLARouter & $\checkmark$ & $\times$ & $\checkmark$ \\
\textbf{Ours} & $\checkmark$ & $\checkmark$ & $\checkmark$ \\
\bottomrule
\end{tabular}

\begin{minipage}{0.96\columnwidth}
\footnotesize
\textit{References: }
1) FrugalGPT~\cite{chen2024frugalgpt};
2) RouteLLM~\cite{ong2025routellm};
3) PersonalizedRouter~\cite{dai2025personalizedrouter};
4) MetaRouter~\cite{zeng2026metarouter};
5) BaRP~\cite{wei2025barp};
6) Conformal Routing~\cite{uddin-bauer-2026-conformal};
7) SLARouter~\cite{woisetschlaeger2026costoptimalllmrouting}.
\end{minipage}
\end{table}

\section{Related Work}
\label{sec:related}

\paragraph{LLM routing and personalization.}
LLM routing reduces inference costs by selecting among models with different capabilities and prices. Early approaches, such as FrugalGPT~\cite{chen2024frugalgpt}, use cascades to balance response quality and cost, while subsequent methods learn routing policies from offline model evaluations, preferences, or task information. Examples include RouteLLM~\cite{ong2025routellm}, RouterDC~\cite{chen2024routerdc}, and CARROT~\cite{somerstep2025carrot}. Benchmarks such as RouterBench~\cite{hu2024routerbench} support systematic evaluation of these approaches. However, conventional routing methods primarily learn cost-effective model selection from available data; they do not jointly learn user preferences online and enforce a chance constraint on realized outcomes.

Personalized routing addresses the additional challenge that different users may prefer different models. PersonalizedRouter~\cite{dai2025personalizedrouter} models heterogeneous user preferences, while GMTRouter~\cite{xie2025gmtrouter} and MetaRouter~\cite{zeng2026metarouter} explore learning-based approaches to personalized routing. BaRP~\cite{wei2025barp} further studies routing under bandit feedback. These methods move closer to our setting, where user preferences are learned through interaction rather than assumed to be known in advance. However, our work differs from these approaches by combining online personalization with an explicit high-probability guarantee on realized cumulative satisfaction.

\paragraph{Online learning with constraints.}
Primal--dual methods provide a general framework for online learning under cumulative constraints. Prior work studies long-term constraints in online convex optimization~\cite{mahdavi2012trading}, constrained Markov decision processes~\cite{cardoso2019large}, and contextual bandits with knapsacks~\cite{agrawal2016linear,badanidiyuru2018bandits,balseiro2020dual}. These methods typically regulate resource consumption or optimize expected performance under cumulative constraints. In contrast, our method uses a virtual queue to track the \emph{realized satisfaction deficit} and learns unknown satisfaction probabilities from contextual bandit feedback. This enables a high-probability guarantee on realized cumulative satisfaction, rather than merely controlling expected performance or constraint violation.

Table~\ref{tab:routing_comparison} summarizes how representative routing approaches compare along three dimensions: online learning, personalization, and explicit satisfaction or safety guarantees. Among the approaches considered, our method is the only one that combines all three.

\section{Problem Formulation}
\label{sec:problem_statement}

We study personalized LLM routing in which a router selects one of
$K$ candidate models for each incoming prompt. The models may differ in
computational cost, and a user's preferred model may depend on the
prompt. Since these preferences are unknown, the router must learn them
from user feedback while keeping inference cost low and maintaining a
desired level of satisfaction.

We formulate this problem as a contextual bandit with generalized
linear feedback and a cumulative chance constraint (also referred to as Quality constraint). The router observes
the prompt and features of all candidate models, but receives
satisfaction feedback only from the selected model.

\subsection{Online Routing Model}

Consider a routing system operating over a horizon of $T$ rounds,
indexed by $t\in[T]:=\{1,\ldots,T\}$. At each round:

\begin{enumerate}
\item \textbf{Context.}
The router observes a prompt $p_t\in\mathcal P$, drawn
independently from the prompt distribution.

\item \textbf{Model selection.}
After observing $p_t$, the router selects
\[
    a_t\in\mathcal A:=\{1,\ldots,K\}.
\]

\item \textbf{Features.}
A known feature map $\phi$ assigns a feature vector to each
prompt-model pair:
\begin{equation}
    x_{t,a}:=\phi(p_t,a)\in\mathbb R^d,
    \qquad a\in\mathcal A.
\end{equation}
Thus, the router can compute the features of all candidate models
before making its selection.

\item \textbf{Cost and feedback.}
Selecting model $a_t$ incurs a known deterministic cost
\begin{equation}
    c(a_t)\in[0,1].
\end{equation}
After the response is generated, the router observes binary
satisfaction feedback $Y_t\in\{0,1\}$, where $Y_t=1$ indicates that
the response satisfies the user's requirements.

\end{enumerate}

The router therefore observes the features of every candidate model but
only observes satisfaction feedback for the selected model.

\subsection{User Satisfaction Model}
\label{sec:model}

We model the probability that a response satisfies the user through a
generalized linear model with a logistic link:

Let $\mathcal G_t$ denote the information available before observing $Y_t$, including the current context and selected action. We assume
\begin{equation}
\begin{gathered}
Y_t\mid\mathcal G_t\sim\operatorname{Bernoulli}(p_{a_t,t}^*),\\
p_{a,t}^*:=\mu(x_{t,a}^{\top}\theta^*),\quad \mu(z):=(1+e^{-z})^{-1}.
\end{gathered}
\end{equation}
Here, $\theta^*\in\mathbb R^d$ is unknown and captures how the
observed prompt-model features relate to the user's satisfaction.

\begin{assumption}[Boundedness]
\label{ass:boundedness}
For every $t\in[T]$ and $a\in\mathcal A$,
\begin{equation}
\|x_{t,a}\|_2\le L,
\qquad
\|\theta^*\|_2\le S,
\end{equation}
where $L,S>0$ are known. Moreover, the model costs are normalized so that they satisfy
$c(a)\in[0,1]$ for all $a\in\mathcal A$.
\end{assumption}

To ensure that the cumulative satisfaction requirement remains feasible, we assume the availability of a reliable fallback model with known satisfaction performance.

\begin{assumption}[Fallback]
\label{ass:slater}
There exists a baseline model $\tilde a\in\mathcal A$ that is exempt from the generalized linear model and has a known satisfaction probability $p^*_{\tilde a,t}=\nu,\ \forall t\in[T]$. We additionally assume $\nu>1-\epsilon$, so the fallback is strictly feasible; its feedback is not used to estimate $\theta^*$.
\end{assumption}

 This assumption does not imply that every LLM routing problem necessarily admits such an action. In practice, the fallback may correspond to a more expensive model that provides sufficiently reliable responses to user queries or a validated service with a known satisfaction guarantee exceeding the target.

We use the following slack factor to quantify this margin.
\begin{definition}[Slack Factor]
For $\epsilon>0$ and fallback satisfaction probability $\nu$, define the slack factor $\xi$ by
    \begin{equation}
        \xi := \frac{1-\nu}{\epsilon}.
    \end{equation}
    The strict-feasibility assumption $\nu\!>\!1-\epsilon$ ensures $0\!\leq\!\xi<1$.
\end{definition}

\subsection{Chance-Constrained Routing Objective}

The router seeks to minimize cumulative inference cost while
maintaining a prescribed level of cumulative user satisfaction. Rather
than constraining only the expected satisfaction, we require the
realized satisfaction over the entire horizon to meet the target with
high probability.

\begin{definition}[Quality Constraint]
\label{def:chance_constraint}
For a target level $1-\epsilon \in (1/2,1)$, failure probability $\delta \in (0,1/2)$, and a known time horizon $T$, the
routing policy must satisfy
\begin{equation}
\mathbb P\left(
\frac{1}{T}\sum_{t=1}^T Y_t
\geq 1-\epsilon
\right)
\geq 1-\delta.
\label{eq:chance_constraint}
\end{equation}
\end{definition}

Equivalently, at least $(1-\epsilon)T$ responses must satisfy the user
with probability at least $1-\delta$. This chance constraint directly captures
the reliability requirement of the routing system.

\subsection{Stationary Contextual Oracle}

To measure the cost of online learning, we compare the router with a
stationary randomized contextual policy that knows the true
satisfaction model.

\begin{definition}[Stationary Contextual Oracle]
\label{def:oracle}
Let $\Pi$ denote the class of stationary randomized contextual policies

$$
\pi:\mathcal P\rightarrow\Delta(\mathcal A),
$$

where $\pi(a\mid p)$ is the probability of selecting model $a$ for
prompt $p$. The oracle is any solution to
\begin{equation}
\pi^*
\in
\arg\min_{\pi\in\Pi}
\mathbb E\left[
\sum_{t=1}^T c(A_t)
\right]
\end{equation}
subject to
\begin{equation}
\mathbb P\left(
\frac{1}{T}\sum_{t=1}^T Y_t(A_t)
\geq 1-\epsilon
\right)
\geq 1-\delta,
\end{equation}
where $A_t\sim\pi^*(\cdot\mid p_t),$ and the probability is taken over the prompts, model selections, and
satisfaction outcomes.
\end{definition}

The oracle represents the minimum expected cost that could be achieved by a stationary
router with complete knowledge of the user's satisfaction model.

\begin{lemma}[Necessary Expectation Condition]
\label{lem:necessary_expectation}
Under the stationary contextual oracle, for
$\epsilon\in(0.01,0.5)$ and $\delta\in(0,0.49)$, there exists a maximum margin
$\rho^*_{\max}(T,\epsilon,\delta) = \mathcal{O}(1/\sqrt{T})$ such that
$$
\mathbb{P}\left(
\frac{1}{T}\sum_{t=1}^T Y_t\geq 1-\epsilon
\right)\geq 1-\delta
\ \Longrightarrow \
\mathbb{E}[p^*_{a_t^*,t}]
\geq 1-\epsilon+\rho^*
$$
for every
$0\leq\rho^*\leq\rho^*_{\max}(T,\epsilon,\delta)$.
\end{lemma}

where $\Phi$ denotes the cumulative distribution function (CDF) of the standard normal distribution, and $\Phi^{-1}$ denotes its inverse (quantile) function.
\begin{remark}
    This lemma shows that, to meet the satisfaction requirement with high probability, the oracle's expected satisfaction must be slightly above the target $1-\epsilon$. Further, the corollary shows that this required gap decreases as $O(1/\sqrt{T})$ and becomes smaller as the horizon $T$ grows.
\end{remark}

\subsection{Cost Regret}

We measure the additional cost incurred while learning relative to the
oracle.

\begin{definition}[Cumulative Cost Regret]
\label{def:regret}
Let $a_t^*$ denote the model selected by the oracle at round $t$. The
expected cumulative cost regret is
\begin{equation}
\operatorname{Regret}_T
=\mathbb E\left[
\sum_{t=1}^T
\bigl(c(a_t)-c(a_t^*)\bigr)
\right].
\end{equation}
\end{definition}

Our goal is therefore to achieve sublinear cost regret while satisfying
the chance constraint in
\eqref{eq:chance_constraint}.

\begin{algorithm}[H]
\caption{Primal-Dual GLM-UCB Router}
\label{alg:primal_dual_glm_ucb_final}
\begin{algorithmic}[1]

\REQUIRE
Time horizon $T$;
action set $\mathcal A$;
parameters $\epsilon,\delta$;
regularization parameter $\beta\geq L^2$; stepsize $\eta$;
safety buffer $\rho$.
\STATE Initialize $
V_1\leftarrow\beta I_d,\ \lambda_1\leftarrow0.
$

\FOR{$t=1,\ldots,T$}

\STATE Observe prompt $p_t$ and compute
$\{x_{a,t}\}_{a\in\mathcal A}$.

\STATE Compute $\widehat\theta_t$ according to
\eqref{eq:exact_mle}.

\FORALL{$a\in\mathcal A$}
    \STATE Compute $\operatorname{UCB}_t(a)$ according to
    \eqref{eq:exact_ucb} with confidence radius ${\alpha_t}$ given by \eqref{eq:alpha_t}. 
\ENDFOR

\STATE Select
\[
a_t
\leftarrow
\arg\min_{a\in\mathcal A}
\left[
    c(a)-\lambda_t\operatorname{UCB}_t(a)
\right].
\]

\STATE Route $p_t$ using model $a_t$ and observe $Y_t$.

\STATE \textbf{if} {$a_t \neq \tilde{a}$}
\[
V_{t+1}
\leftarrow
V_t+x_{a_t,t}x_{a_t,t}^{\top}.
\]
\STATE \textbf{else}
\STATE $V_{t+1}\leftarrow V_t$.
\STATE \textbf{end if}
\STATE Update
\[
\lambda_{t+1}
\leftarrow
\max\big\{
0,\,
\lambda_t+
\eta(1-\epsilon+\rho-Y_t)
\big\}.
\]

\ENDFOR

\end{algorithmic}
\end{algorithm}

\section{Primal-Dual GLM-UCB Router}
\label{sec:algorithm}

The router must address several challenges simultaneously. It must learn which models are likely to satisfy the user while preventing cumulative satisfaction from falling below its target and keeping the cost as low as possible. 

We now propose our main algorithm that achieves sublinear regret while satisfying the chance constraint~\eqref{eq:quality_constraint_intro}. The algorithm proceeds in the following steps:
\begin{enumerate} 
\item At each round, the router first uses the observed feedback from previous rounds to estimate the unknown GLM parameter $\theta^*$. 
\item Next, it selects a model by balancing its inference cost against its estimated satisfaction, with the tradeoff determined by the current dual variable.
\item After observing the realized satisfaction, the router updates the dual variable according to the resulting satisfaction deficit. 
\end{enumerate}
The overall algorithm is summarized in Algorithm~\ref{alg:primal_dual_glm_ucb_final}.

\subsection{Learning User Satisfaction}

The router estimates $\theta^*$ using all feedback observed so far. At round $t$, it computes the constrained regularized
logistic estimator
\begin{equation}
\widehat\theta_t
=
\argmin_{\theta\in\mathcal B_2(S)}
\sum_{\substack{\tau=1\\a_\tau \neq \tilde{a}}}^{t-1}
l\left(
Y_\tau,
\mu(x_{a_\tau,\tau}^{\top}\theta)
\right)
+
\frac{\beta}{2}\|\theta\|_2^2
,
\label{eq:exact_mle}
\end{equation}
where $l(y,p)
:=-y\log p-(1-y)\log(1-p)$ and
\begin{equation}
\mathcal B_2(S)
=\{\theta\in\mathbb R^d:\|\theta\|_2\le S \}.
\end{equation}

After observing the selected model's feedback, the design matrix is
updated if $a_t \neq \tilde{a}$ according to
\begin{equation}
V_{t+1}
=V_t+x_{a_t,t}x_{a_t,t}^{\top}.
\end{equation}

The parameter constraint ensures that both the true and estimated
parameters lie in the same bounded domain. Since the logistic link has
positive derivative on any bounded interval, this gives the uniform
curvature needed for the confidence analysis.

\subsection{Optimistic Satisfaction Estimates}

For every candidate model, we form
\begin{equation}
\operatorname{UCB}_t(a)
=\min\{
1,
\mu(x_{a,t}^{\top}\widehat\theta_t)
+
\alpha_t\|x_{a,t}\|_{V_t^{-1}}
\}.
\label{eq:exact_ucb}
\end{equation}
Here, $\alpha_t$ is the confidence radius defined by
\begin{equation}
\label{eq:alpha_t}
\alpha_t
=\frac{1}{\kappa}
\left[
\frac{1}{2}
\sqrt{
d\log\left(
1+\frac{tL^2}{d\beta}
\right)
+
2\log\left(\frac{2}{\delta_{\mathcal E}}\right)
}
+
\sqrt{\beta}S
\right],
\end{equation}
where $\kappa$ is a constant defined by,\footnote{We note that under Assumption~\ref{ass:boundedness}, all relevant linear predictors
satisfy $|x_{t,a}^{\top}\theta|\leq LS$.}
\begin{align}
\kappa
&:=
\inf_{|z|\leq LS}\mu'(z)>0.
\label{eq:kappa} \\
\delta_{\mathcal{E}} &:=\frac{1}{T^\zeta} \ , \ \zeta \in (2.5,\infty).\label{eq: delta_E}
\end{align}
Moreover, for $a=\tilde{a}$, we set
\begin{equation}
\operatorname{UCB}_t(\tilde{a}) = \nu
\end{equation}
The confidence-based formula is used for $a\neq\tilde a$; the fallback is handled separately by its known value $\nu$. Thus, the router can explore potentially useful models
without treating an uncertain model as poor simply because it has not
yet received enough feedback.

\begin{figure*}[t]
    \centering
    \includegraphics[width=0.32\textwidth]{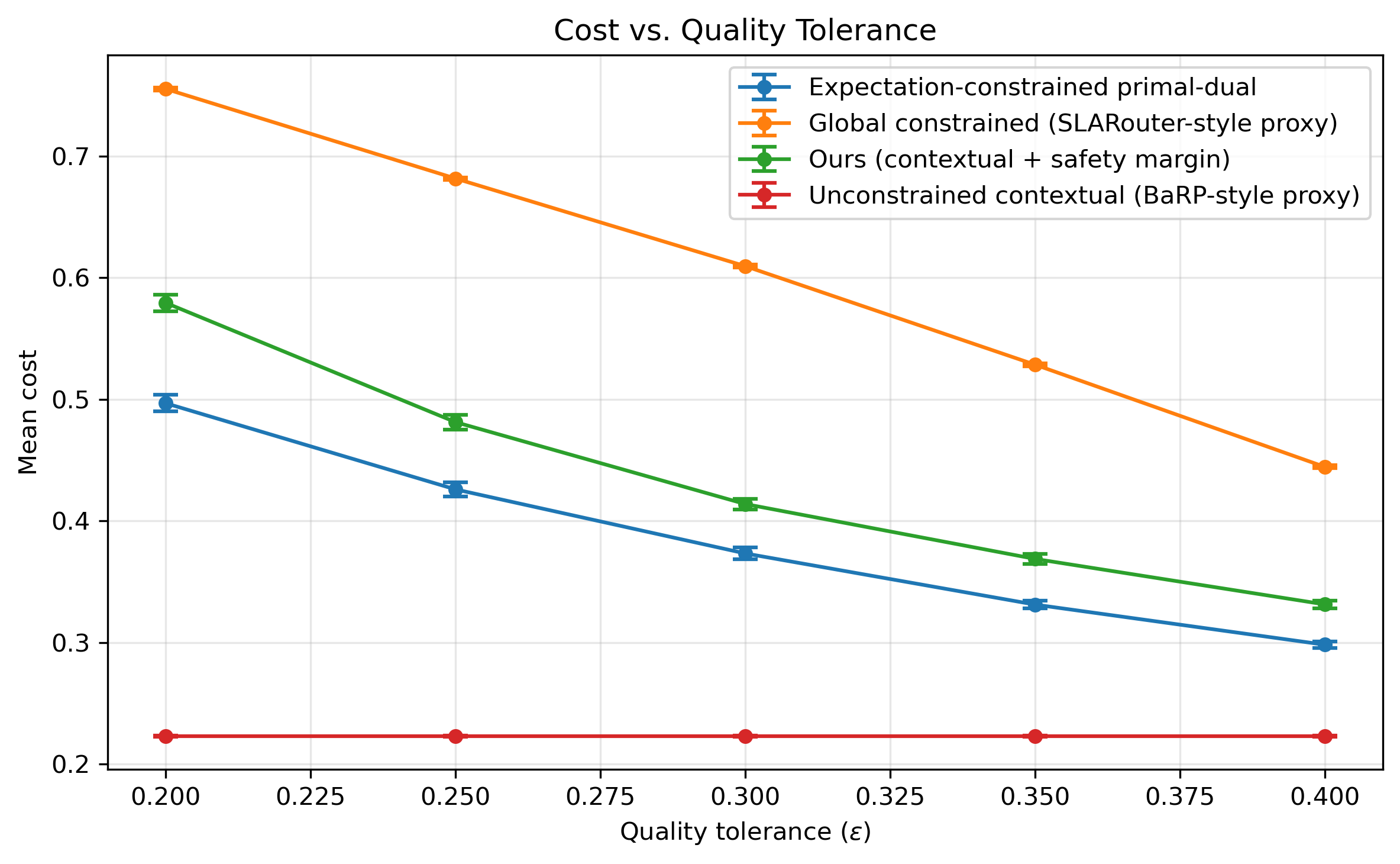}
    \hfill
    \includegraphics[width=0.32\textwidth]{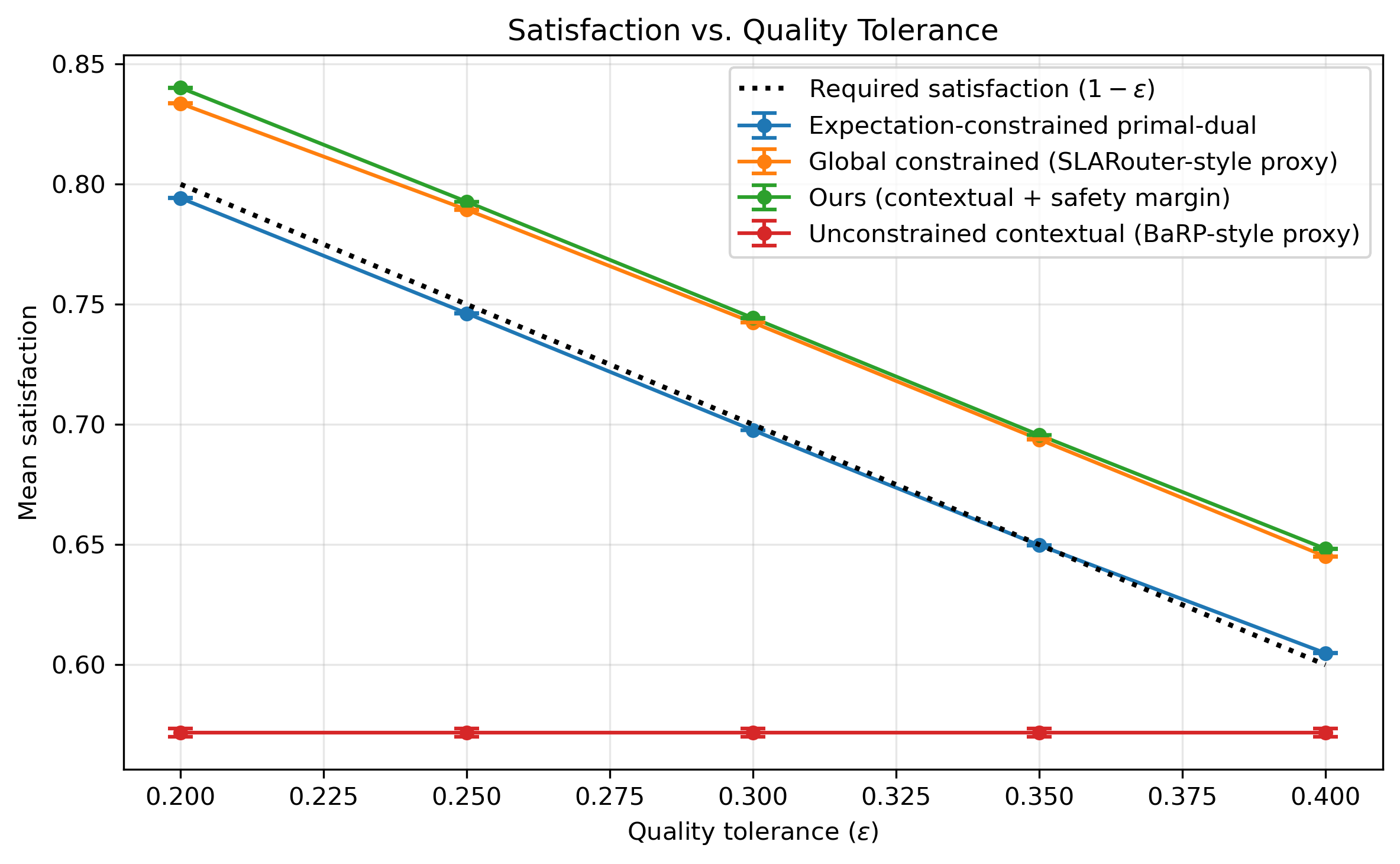}
    \hfill
    \includegraphics[width=0.32\textwidth]{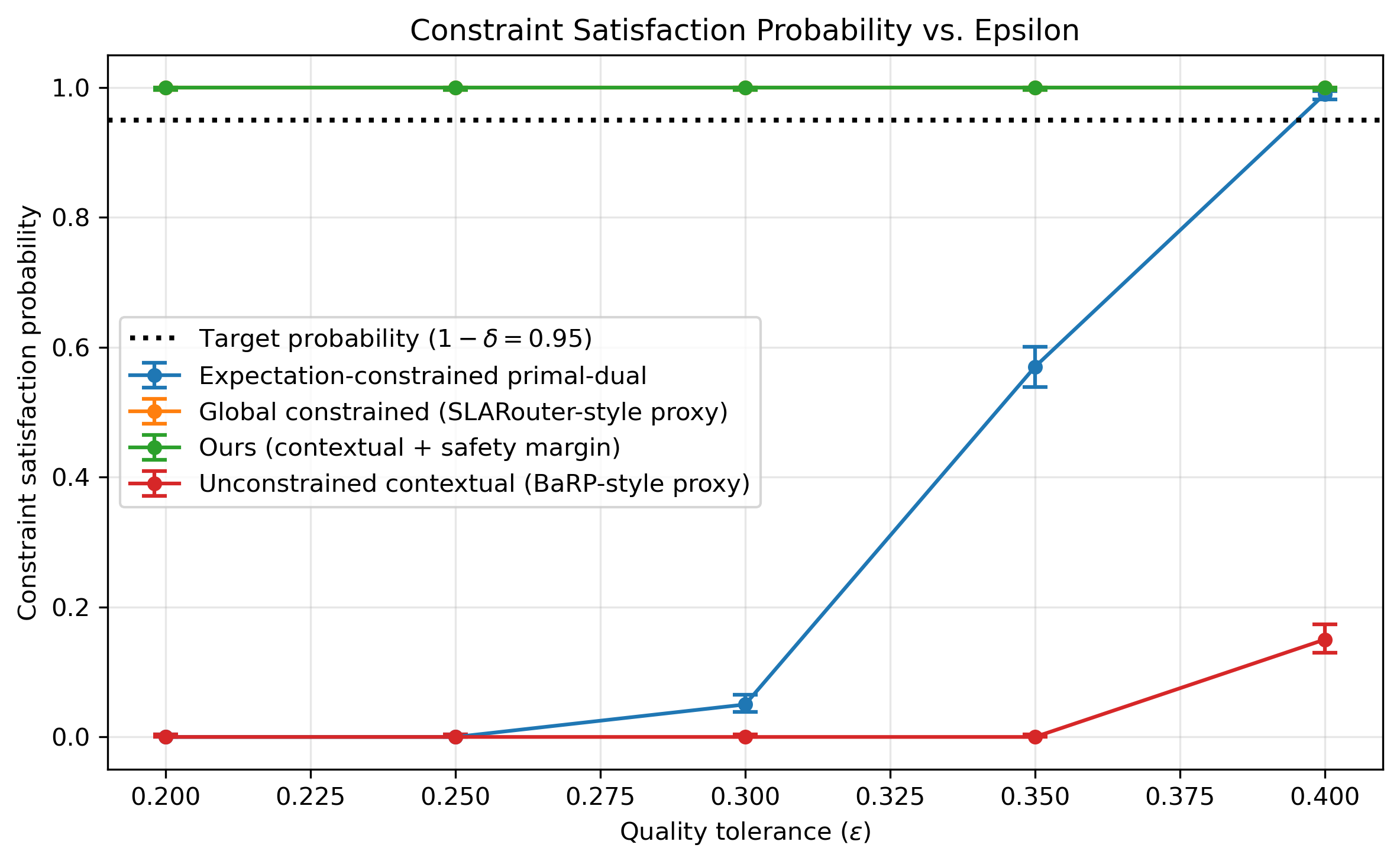}
    \caption{Comparison of the proposed routing algorithm with alternative routing strategies as the satisfaction tolerance $\epsilon$ varies. In the third panel, the green and orange curves overlap.}
    \label{fig:baseline_comparison}
\end{figure*}
\subsection{Model Selection and Virtual Satisfaction Queue}

To balance inference cost against the cumulative satisfaction requirement,
we maintain a nonnegative virtual satisfaction queue
$\lambda_t \geq 0$, initialized with $\lambda_1=0$. The variable
$\lambda_t$ can equivalently be viewed as a dual variable associated
with the satisfaction requirement. After selecting model $a_t$ and
observing the realized satisfaction $Y_t$, the queue is updated according to
\begin{equation}
\lambda_{t+1}
=
\max\big\{
0,\,
\lambda_t
+
\eta\left(1-\epsilon+\rho-Y_t\right)
\big\},
\label{eq:exact_dual_update}
\end{equation}
where $\eta>0$ is the step size and $\rho>0$ is a safety buffer.
Thus, the queue increases when the realized satisfaction falls below
the buffered target $1-\epsilon+\rho$ and decreases otherwise.

At round $t$, given the current queue $\lambda_t$, the router selects
the model according to
\begin{equation}
a_t
=
\arg\min_{a\in\mathcal A}
\left[
c(a)-\lambda_t\operatorname{UCB}_t(a)
\right].
\label{eq:exact_primal}
\end{equation}
The queue $\lambda_t$ therefore determines how strongly the router
prioritizes satisfaction relative to cost. When $\lambda_t$ is small,
the router places greater emphasis on cost, whereas a larger
$\lambda_t$ favors models with higher optimistic satisfaction
estimates. 

We choose
\begin{equation}
\eta=\frac{C_\eta}{\sqrt{T}},\qquad
\rho=\frac{C_\rho}{\sqrt{T}}
\label{eq:eta_rho}
\end{equation}
where
$C_\eta = \frac{1}{(1-\xi)\epsilon \sqrt{2 \log(4/\delta)}}$ and $C_\rho =5\sqrt{2 \log(4/\delta)}$. The safety buffer $\rho$
provides additional slack to account for the random deviation between
expected and realized cumulative satisfaction.


\section{Theoretical Guarantees}
\label{sec:theoretical_guarantees}

We now establish theoretical guarantees for the algorithm, including a sublinear regret bound. We first derive a high-probability confidence bound for the estimated satisfaction model, which is then used to control the $\operatorname{UCB}_t(a)$ terms. Building on these results, we analyze the evolution of the virtual queue and establish the main theoretical guarantees of the algorithm.

We begin with the following lemma, which provides a high-probability bound on the estimation error of $\theta^*$.

\begin{lemma}[Uniform GLM Confidence Bound]
\label{lem:confidence}
Under Assumption~\ref{ass:boundedness}, with probability at least
$1-\delta_{\mathcal E}$, simultaneously for all $t\in[T]$ and
$a\in\mathcal A\setminus\{\tilde a\}$,
\begin{equation}
\|\widehat\theta_t-\theta^*\|_{V_t}\leq\alpha_t
\label{eq:parameter_confidence_bound}
\end{equation}
and
\begin{equation}
\left|
\mu(x_{a,t}^{\top}\widehat\theta_t)
-
\mu(x_{a,t}^{\top}\theta^*)
\right|
\leq
\alpha_t\|x_{a,t}\|_{V_t^{-1}}.
\label{eq:link_error_bound}
\end{equation}
\end{lemma}

Next, the following lemma establishes that the router's satisfaction estimate is optimistic, providing an upper bound on the true satisfaction probability. As more feedback is collected, the estimate becomes increasingly accurate, which plays a key role in the subsequent regret analysis.
\begin{lemma}[UCB Bounds]
\label{lem:ucb_bounds}
On the confidence event of Lemma~\ref{lem:confidence}, for every
$t\in[T]$ and $a\in\mathcal A\setminus\{\tilde a\}$,\footnote{For the fallback action, the UCB is set to its known satisfaction probability $\nu$.}
\begin{enumerate}
\item The UCB is optimistic:
\begin{equation}
p_{a,t}^*
\leq
\operatorname{UCB}_t(a).
\label{eq:ucb_optimism}
\end{equation}

\item The optimism gap satisfies
\begin{equation}
\operatorname{UCB}_t(a)-p_{a,t}^*
\leq
2\alpha_t\|x_{a,t}\|_{V_t^{-1}}.
\label{eq:ucb_upper_deviation}
\end{equation}
\end{enumerate}
\end{lemma}

\subsection{Control of the Satisfaction Queue}

Next, we show that the virtual satisfaction queue remains controlled with high probability.

\begin{lemma}[High-Probability Queue Bound]
\label{lem:dual_bound}
Suppose Assumptions~\ref{ass:boundedness} and~\ref{ass:slater} hold.
For any $\delta\in(0,1)$, there exist
$$
B_\lambda=\frac{6}{(1-\xi)\epsilon}+\widetilde{\mathcal O}(1/\sqrt T)
$$
such that, with probability at least $1-\delta$,
\begin{equation}
\max_{1\leq t\leq T+1}\lambda_t
\leq
B_\lambda.
\label{eq:dual_bound_main}
\end{equation}
\end{lemma}

The queue bound in \eqref{eq:dual_bound_main} establishes that the dual queue remains bounded throughout the execution of the algorithm. This result is then used to show that the desired quality constraint is satisfied with high probability. 
\begin{figure*}[t]
    \centering
    \includegraphics[width=0.32\textwidth]{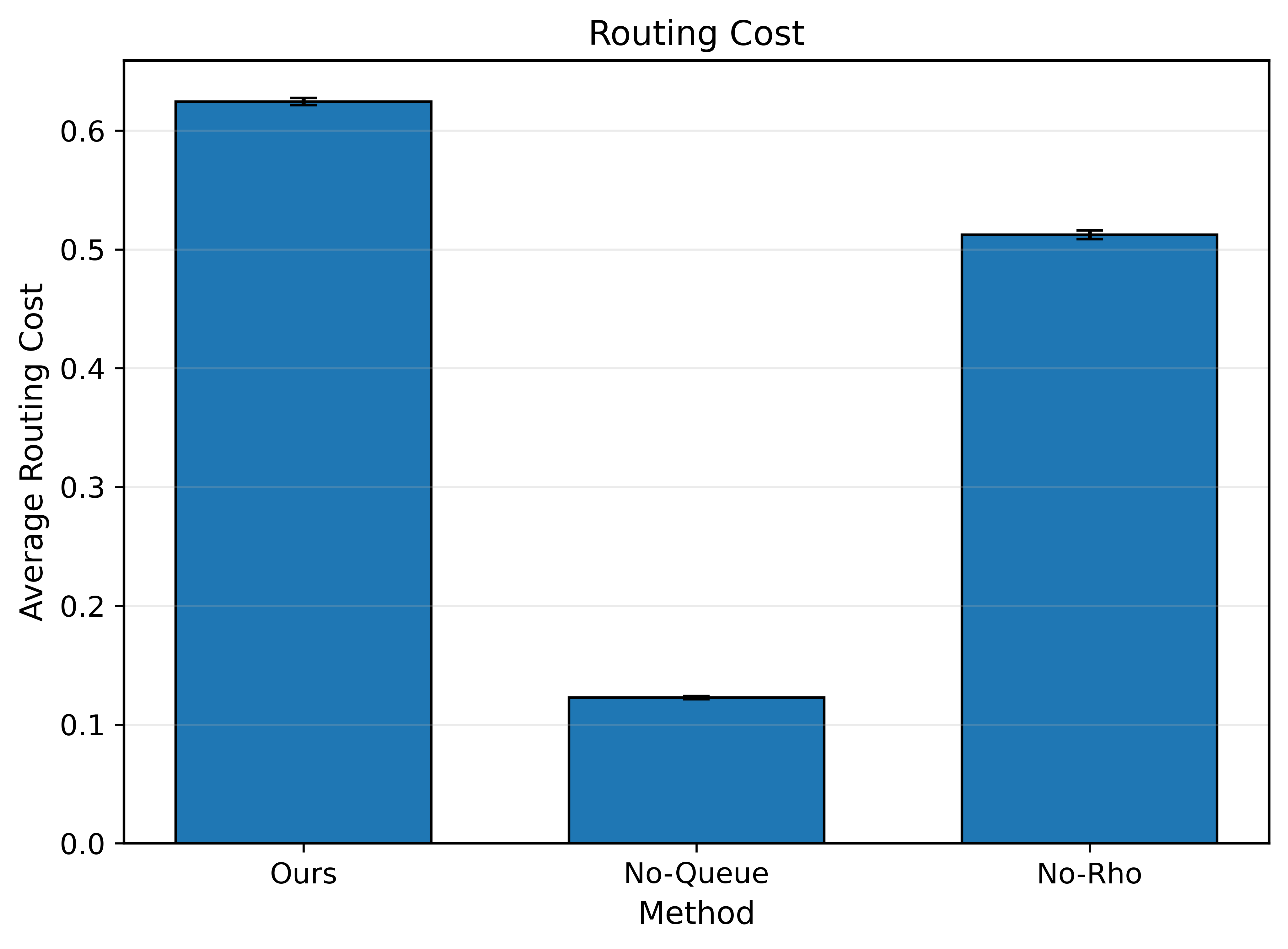}
    \hfill
    \includegraphics[width=0.32\textwidth]{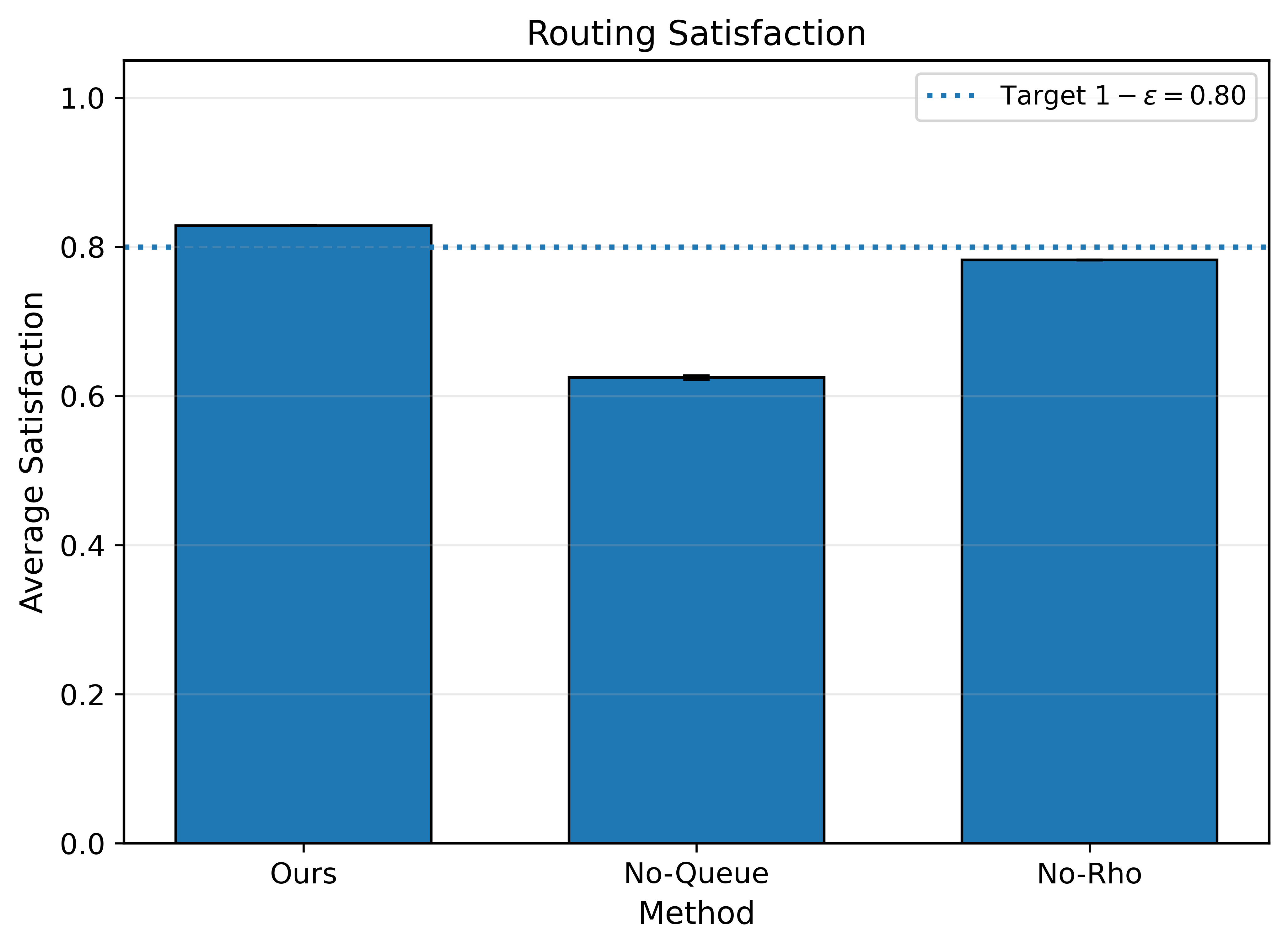}
    \hfill
    \includegraphics[width=0.32\textwidth]{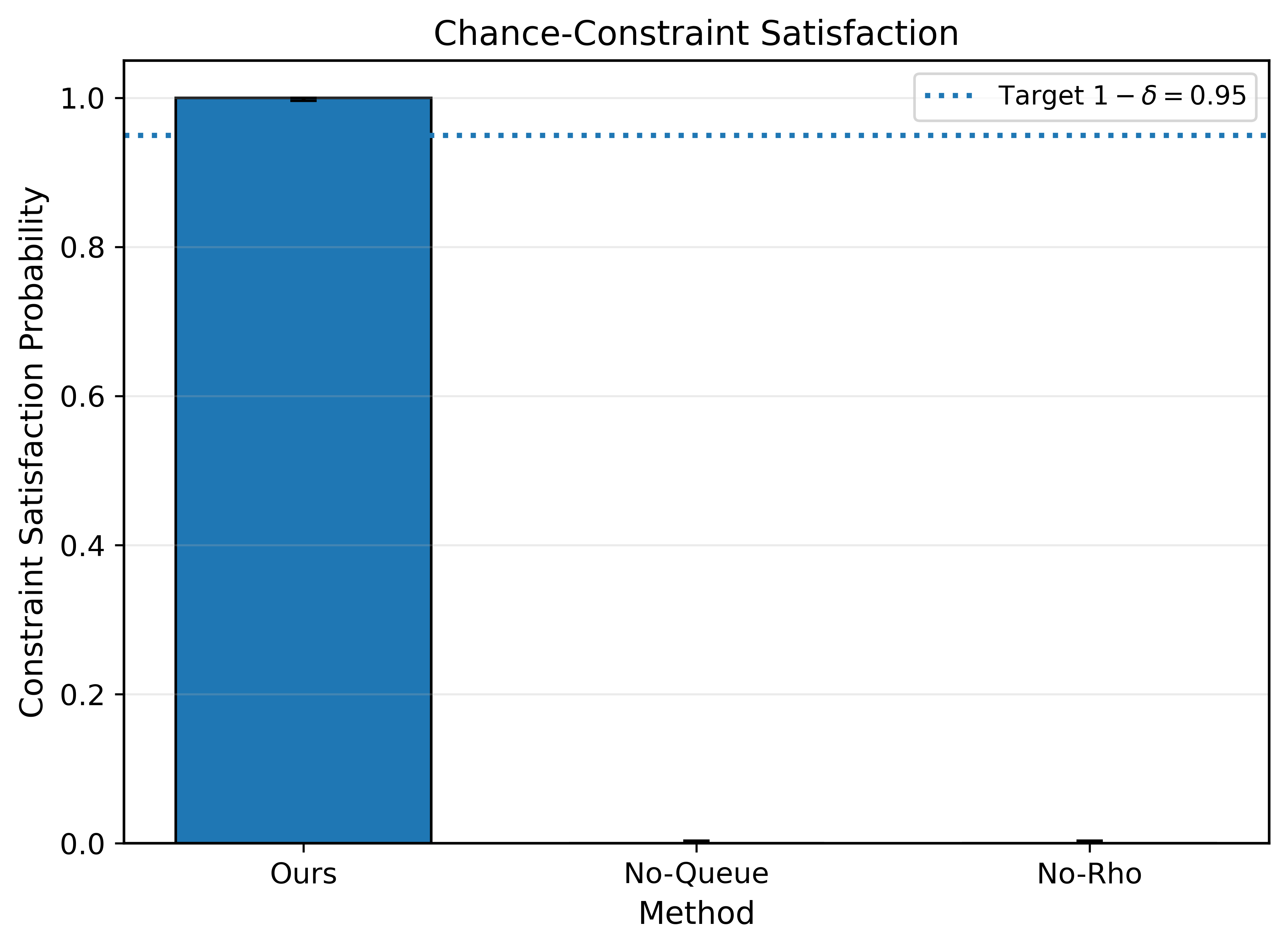}

    \caption{Overall performance and ablation results. Error bars denote $95\%$ confidence intervals. The dotted lines indicate the required satisfaction level $1-\epsilon=0.8$ and constraint probability $1-\delta=0.95$, respectively.}
    \label{fig:ablation}
\end{figure*}

\subsection{Main Guarantee}

We can now state the main result.

\begin{theorem}[Cost Regret and Satisfaction Guarantee]
\label{thm:cost_regret}
Suppose Assumptions~\ref{ass:boundedness} and~\ref{ass:slater} hold, $\xi < 1$ and
the constants in Algorithm~\ref{alg:primal_dual_glm_ucb_final} are chosen
as specified in Section~\ref{sec:algorithm}. Then, there exists a constant $T_0$ such that for $T \geq T_0$ the proposed router
satisfies
\begin{equation}
\mathbb P\left(
\sum_{t=1}^T Y_t
\geq
(1-\epsilon)T
\right)
\geq
1-\delta,
\label{eq:sla_feasibility_statement}
\end{equation}
and its expected cumulative cost regret relative to the optimal stationary
contextual oracle satisfies
\begin{equation}
\operatorname{Regret}_T
\leq
\widetilde{\mathcal O}\!\left(
\frac{(d+\sqrt{\log(4/\delta)})\sqrt{T}}
{\kappa(1-\xi)\epsilon}
\right)
\label{eq:regret_bound_statement}
\end{equation}
\end{theorem}

\begin{proof}[Proof Sketch]
The proof has two main steps. First, Lemmas~\ref{lem:confidence} and~\ref{lem:ucb_bounds} establish the accuracy of the router's optimistic satisfaction estimates, with estimation errors decreasing as more feedback is collected. Under the decision rule in Eq.~\eqref{eq:exact_primal}, the router balances cost against these estimates, allowing the cost regret to be bounded in terms of the estimation errors weighted by the virtual queue $\lambda_t$. Lemma~\ref{lem:dual_bound} plays a crucial role by establishing a bound on the queue, thereby limiting the impact of estimation errors on regret. Summing these errors over time yields the regret bound in Eq.~\eqref{eq:regret_bound_statement}. Second, the queue update in Eq.~\eqref{eq:exact_dual_update} tracks accumulated satisfaction deficits and increases the queue when realized satisfaction falls below the buffered target. The boundedness of the queue, together with the safety buffer $\rho$ and a concentration argument, ensures that the cumulative satisfaction requirement in Eq.~\eqref{eq:sla_feasibility_statement} holds with probability at least $1-\delta$. The detailed proofs are provided in Appendices~\ref{sec:regret_proof_detailed} and~\ref{sec:quality_bound}.
\end{proof}
Theorem~\ref{thm:cost_regret} shows that the router can learn user preferences while maintaining the required satisfaction level, incurring only sublinear additional cost relative to an oracle. Moreover, the dependence of our regret bounds on the slack factor $\xi$ indicates that achieving lower routing costs becomes increasingly challenging as the fallback performance $\nu$ approaches the target satisfaction level $1-\epsilon$.

\section{Experimental Results}
We evaluate the proposed algorithm through five experiments and a real-LLM case study, examining routing baselines, ablations, real-world performance, personalization, and scalability with the decision horizon. Details are provided in Appendix~\ref{section:expts_add}.

\begin{figure*}[t]
    \centering
    \includegraphics[width=0.45\textwidth]{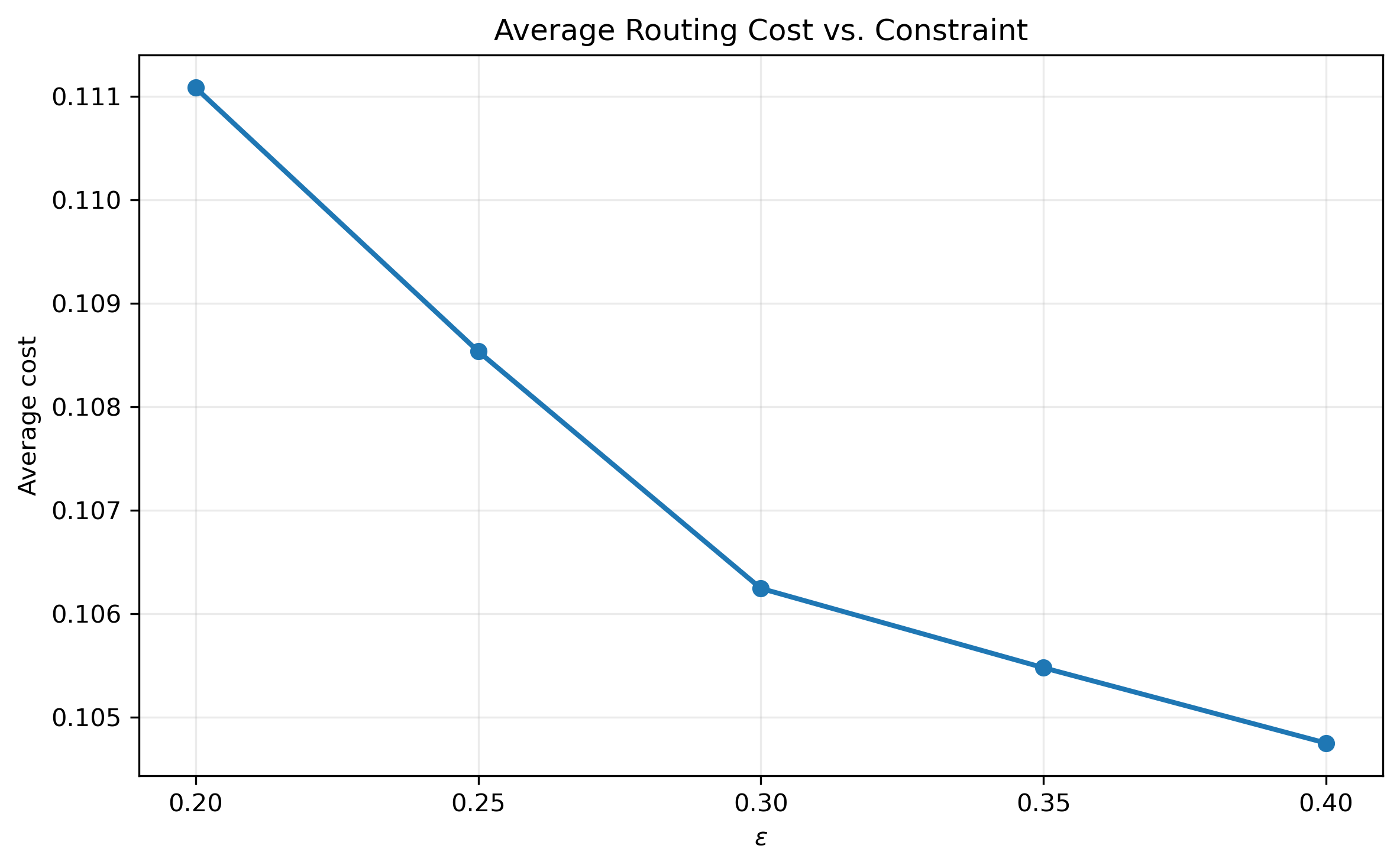}
    \hfill
    \includegraphics[width=0.53\textwidth]{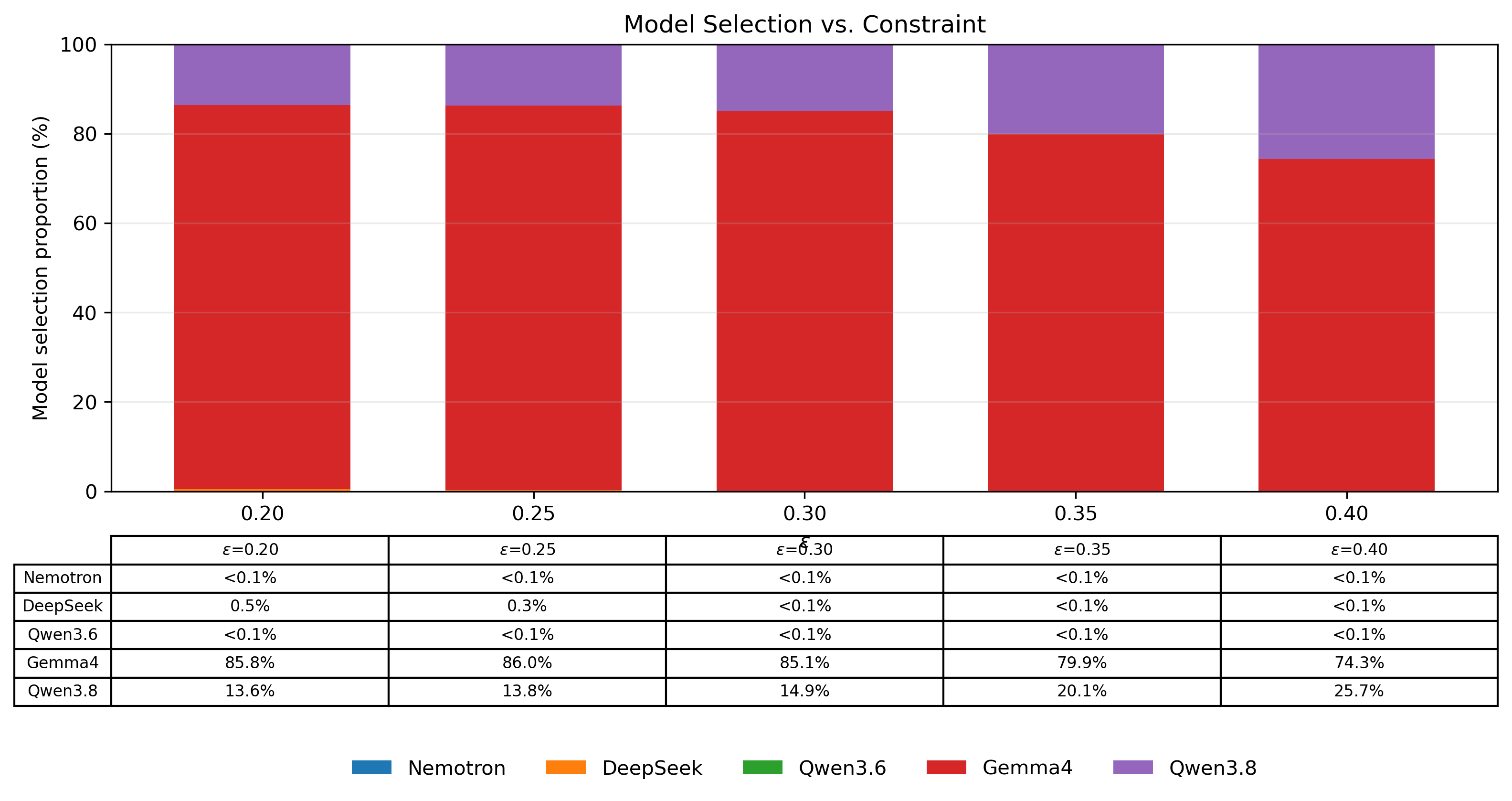}
    \caption{Real-LLM routing POC: (a) average routing cost as a function of $\epsilon$ and (b) model-selection composition as a function of $\epsilon$.}
    \label{fig:real_llm}
\end{figure*}
\begin{figure}[t]
    \centering
    \includegraphics[
        width=0.75\linewidth,
        trim=0 40 0 0,
        clip
    ]{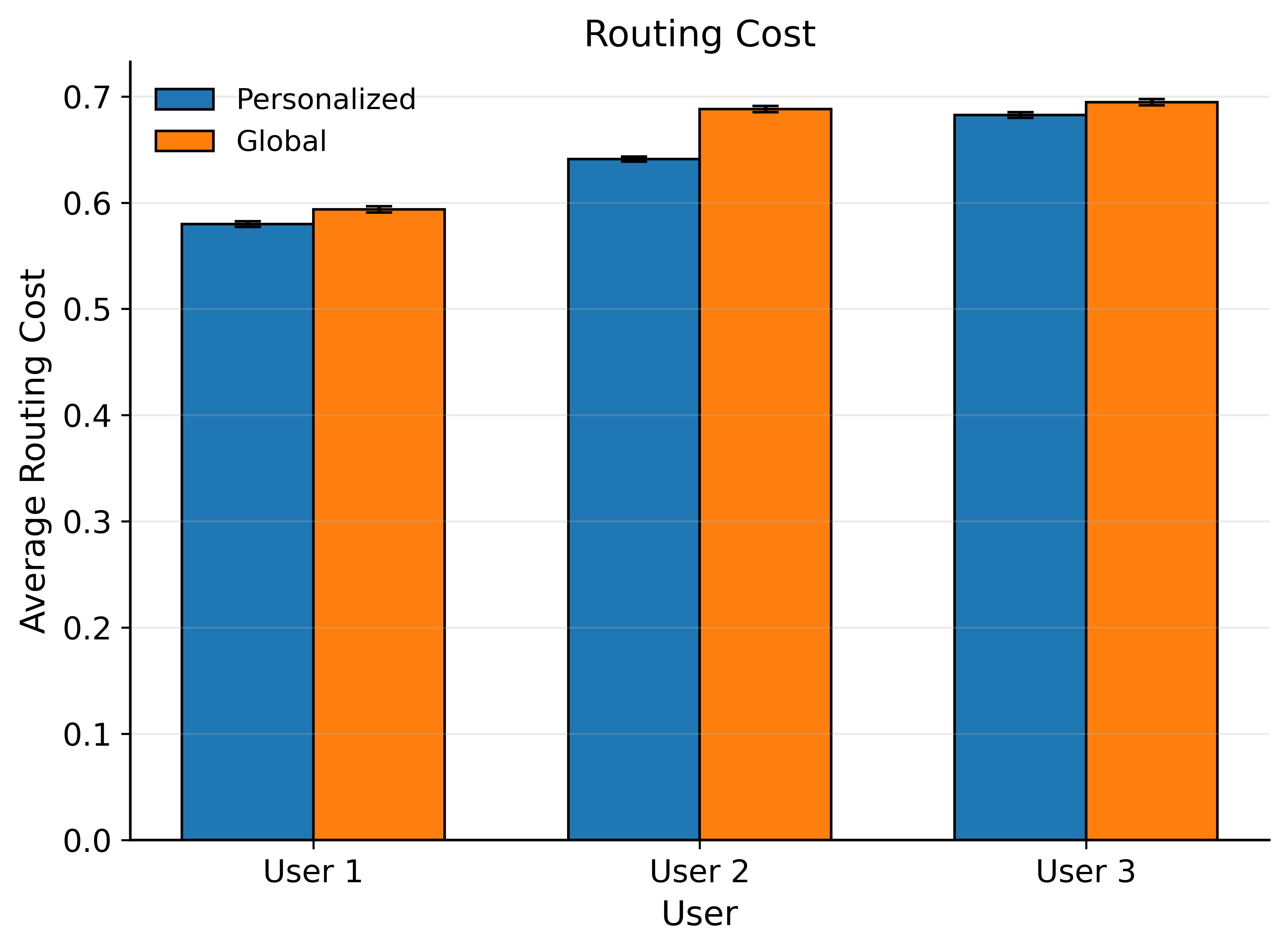}
    \caption{Personalized and global routing costs for Users 1--3.}
    \label{fig:personalization}
\end{figure}

\begin{figure}[t]
    \centering
    \includegraphics[
        width=0.75\linewidth,
        trim=0 40 0 0,
        clip
    ]{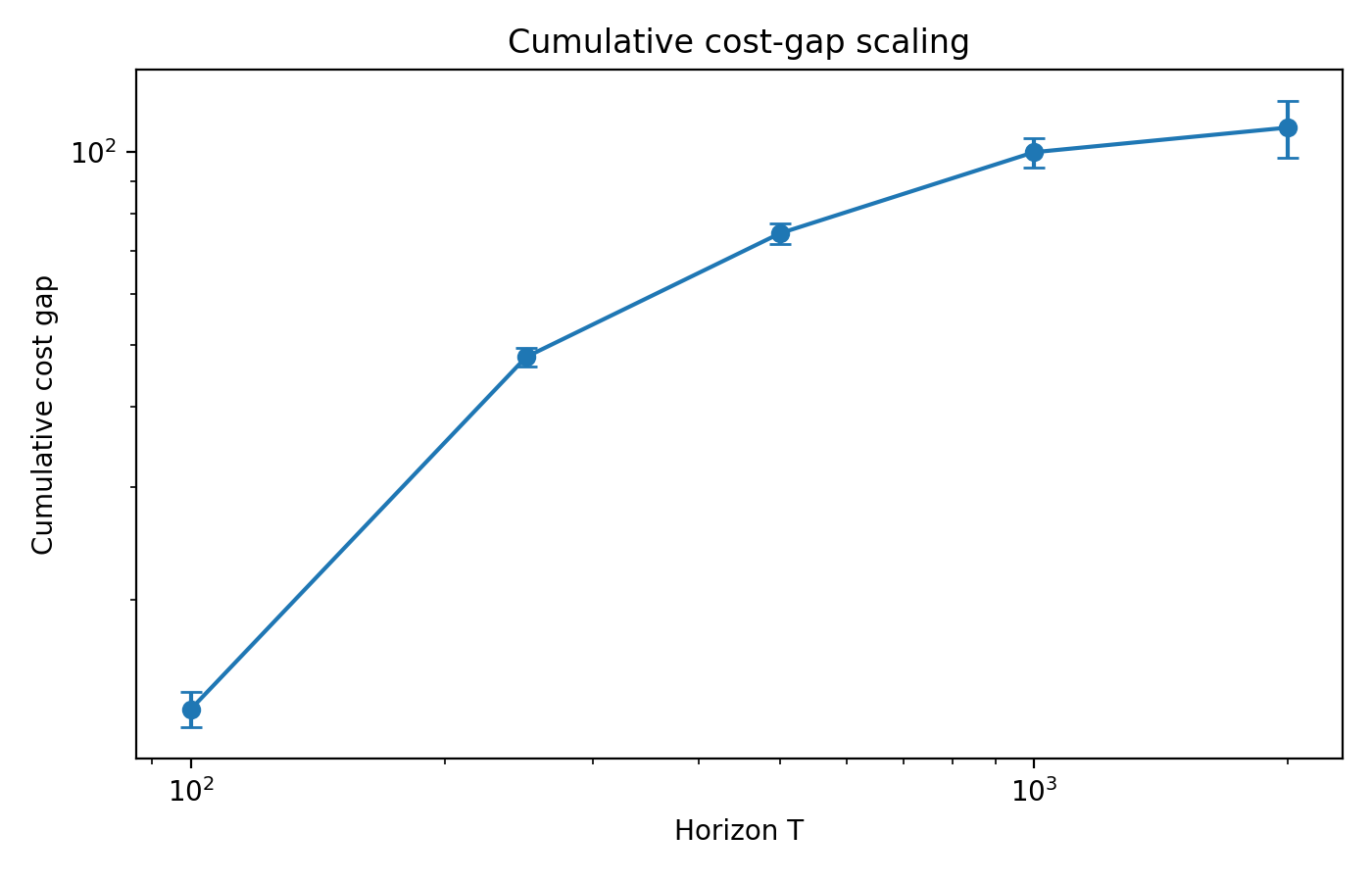}
    \caption{Per-round and cumulative cost gap relative to the oracle across horizons.}
    \label{fig:horizon}
\end{figure}

\subsection{Comparison with Alternative Routing Strategies}

We first compare the proposed method with three representative routing strategies: a SLARouter-style global constrained router~\cite{woisetschläger2026costoptimalllmroutinglimited}, a BaRP-style unconstrained personalized router~\cite{wei2025learning}, and an expectation-constrained primal-dual router ~\cite{yang2025data}. We vary the satisfaction tolerance $\epsilon$ while fixing $\delta=0.05$, $T=1000$, and all methods are evaluated over the same $1000$ replay trials.

Figure~\ref{fig:baseline_comparison} shows the resulting cost, average satisfaction, and empirical constraint satisfaction probability. As $\epsilon$ increases, the proposed method reduces its average routing cost because the required satisfaction level $1-\epsilon$ becomes less stringent. At the same time, its empirical constraint satisfaction probability remains close to the target $1-\delta$. We observe that the SLARouter does not exploit contextual information and therefore has a higher average cost, while the BaRP-style router and the expectation-constrained primal-dual router fail to satisfy the chance constraint. 

\subsection{Ablation}

Next, we examine the contributions of the dual queue and safety margin in the proposed algorithm. We compare \textbf{Ours} with \textbf{No-Queue}, which removes the $\lambda$ update equation, and \textbf{No-$\rho$}, which sets the safety margin to $\rho=0$. All methods are evaluated with $\epsilon=0.2$ and $\delta=0.05$. We report the average routing cost, the average satisfaction, and the probability of empirical constraint satisfaction with $95\%$ confidence intervals.

As shown in Figure ~\ref{fig:ablation}, \textbf{Ours} satisfies the empirical chance constraint in almost all trials. \textbf{No-Queue} and \textbf{No-$\rho$} achieve lower routing costs, but their empirical constraint-satisfaction probabilities are close to zero. These results suggest that the dual queue helps adapt routing decisions to accumulated constraint violations, while the safety margin provides additional protection against uncertainty in the learned satisfaction estimates.

\subsection{Real-LLM Routing}
\label{section:poc}

We next evaluate the proposed method on real LLMs using 100 queries from WildBench v2 and five available LLMs \footnote{Additional details are in Appendix \ref{section:expts_add}} as routing actions. For each query--model pair, the response is represented by a six-dimensional quality vector measuring correctness, relevance, completeness, clarity, conciseness, and safety, evaluated by a lightweight judge model. We then apply a fixed parameter $\theta^*$ in the logistic model of Section~\ref{sec:model} to generate simulated satisfaction outcomes.

Figure~\ref{fig:real_llm} illustrates the routing cost and model selection as the satisfaction tolerance $\epsilon$ varies. The average routing cost decreases as $\epsilon$ increases, since it permits more frequent selection of inexpensive models. In particular, the router increasingly selects Qwen3.8-27B, the lowest-cost model, as the constraint is relaxed. Conversely, for smaller $\epsilon$, it places greater weight on higher-satisfaction models such as DeepSeek-V4-Flash despite their higher cost. This illustrates how the routing policy adapts model selection to the quality--cost trade-off imposed by the chance constraint.

\subsection{Personalization}

We next evaluate whether explicitly modeling user-specific preferences can improve routing efficiency. We consider three users who receive the same sequence of requests but have different preference parameters. We compare a personalized learner, which maintains separate model estimates and safety queues for each user, with a global learner that pools observations across users. Both methods use the same requests and counterfactual outcomes. 

Figure~\ref{fig:personalization} reports the average routing cost for each user. The personalized learner achieves lower routing cost while maintaining satisfaction above the required level. We omit a separate satisfaction plot because all methods satisfy the constraint in this experiment; therefore, the cost comparison more clearly illustrates the benefit of personalization.

\subsection{Finite-Horizon Scaling}

Finally, we examine how the proposed method performs as the number of routing decisions increases. We consider horizons $T\in\{100,250,500,1000,2000\}$ and compare the cumulative cost of the proposed method with that of an oracle policy that knows the true model parameters.

As shown in Fig.~\ref{fig:horizon}, the cumulative cost gap relative to the oracle grows sublinearly over the evaluated horizons. This empirical behavior is consistent with the sublinear regret guarantee established in our analysis. Satisfaction results are omitted for brevity.

\section{Conclusion}
\label{sec:conclusion}

We presented a framework for online personalized LLM routing that learns user preferences, minimizes inference costs, and enforces a high-probability quality constraint. Our primal-dual GLM-UCB algorithm combines optimistic preference learning with a satisfaction-deficit queue and achieves sublinear cost regret while ensuring satisfaction with high probability. Experimental results demonstrate that the proposed method can reduce routing costs through personalization while maintaining the required quality level, including in a real-world LLM routing task.

Future research directions include accommodating nonstationary user preferences and incorporating richer feedback signals. Additionally, the current analysis assumes a uniform fallback satisfaction level. Extending the analysis to accommodate context-dependent or uncertain fallback guarantees is another interesting direction for future work.

\bibliography{refs}

@inproceedings{agrawal2016linear, title={Linear Contextual Bandits with Knapsacks}, author={Agrawal, Shipra and Devanur, Nikhil R.}, booktitle={Advances in Neural Information Processing Systems}, volume={29}, year={2016} }

@article{chen2024frugalgpt, title={FrugalGPT: How to Use Large Language Models While Reducing Cost and Improving Performance}, author={Chen, Lingjiao and Zaharia, Matei and Zou, James}, journal={Transactions on Machine Learning Research}, year={2024} }

@article{hu2024routerbench, title={RouterBench: A Benchmark for Multi-LLM Routing System}, author={Hu, Qitian Jason and Bieker, Jacob and Li, Xiuyu and Jiang, Nan and Keigwin, Benjamin and Ranganath, Gaurav and Keutzer, Kurt and Upadhyay, Shriyash Kaustubh}, journal={arXiv preprint arXiv:2403.12031}, year={2024} }

@inproceedings{ong2025routellm, title={RouteLLM: Learning to Route LLMs from Preference Data}, author={Ong, Isaac and Almahairi, Amjad and Wu, Vincent and Chiang, Wei-Lin and Wu, Tianhao and Gonzalez, Joseph E. and Kadous, M. Waleed and Stoica, Ion}, booktitle={International Conference on Learning Representations}, year={2025} }

@article{somerstep2025carrot, title={CARROT: A Cost Aware Rate Optimal Router}, author={Somerstep, Seamus and Polo, Felipe Maia and Melo de Oliveira, Allysson Flavio and Mangal, Prattyush and Silva, M{\'\i}rian and Bhardwaj, Onkar and Yurochkin, Mikhail and Maity, Subha}, journal={arXiv preprint arXiv:2502.03261}, year={2025} }

@article{dai2025personalizedrouter, title={PersonalizedRouter: Personalized LLM Routing via Graph-based User Preference Modeling}, author={Dai, Zhongjie and Feng, Tao and You, Jiaxuan}, journal={arXiv preprint arXiv:2511.16883}, year={2025} }

@inproceedings{li2010contextual,
  title={A contextual-bandit approach to personalized news article recommendation},
  author={Li, Lihong and Chu, Wei and Langford, John and Schapire, Robert E.},
  booktitle={Proceedings of the 19th International Conference on World Wide Web},
  pages={661--670},
  year={2010}
}

@inproceedings{xu2020timevarying,
  title={Contextual-Bandit Based Personalized Recommendation with Time-Varying User Interests},
  author={Xu, Xiao and Dong, Fang and Li, Yanghua and He, Shaojian and Li, Xin},
  booktitle={Proceedings of the AAAI Conference on Artificial Intelligence},
  volume={34},
  pages={6518--6525},
  year={2020}
}

@article{xie2025gmtrouter,
  title={GMTRouter: Personalized LLM Router over Multi-turn User Interactions},
  author={Xie, Encheng and Sun, Yihang and Feng, Tao and You, Jiaxuan},
  journal={arXiv preprint arXiv:2511.08590},
  year={2025}
}

@misc{zeng2026metarouter,
      title={Learning to Route LLMs from Implicit Cost-Performance Preferences via Meta-Learning}, 
      author={Jiahao Zeng and Ming Tang and Ningning Ding},
      year={2026},
      eprint={2606.06178},
      archivePrefix={arXiv},
      primaryClass={cs.LG},
      url={https://arxiv.org/abs/2606.06178}, 
}

@article{mahdavi2012trading,
  title={Trading Regret for Efficiency: Online Convex Optimization
         with Long Term Constraints},
  author={Mahdavi, Mehrdad and Jin, Rong and Yang, Tianbao},
  journal={Journal of Machine Learning Research},
  volume={13},
  number={81},
  pages={2503--2528},
  year={2012}
}

@inproceedings{cardoso2019large,
  title={Large Scale Markov Decision Processes with Changing Rewards},
  author={Cardoso, Adrian Rivera and Wang, He and Xu, Huan},
  booktitle={Advances in Neural Information Processing Systems},
  volume={32},
  year={2019}
}

@article{badanidiyuru2018bandits,
  title={Bandits with Knapsacks},
  author={Badanidiyuru, Ashwinkumar and Kleinberg, Robert
          and Slivkins, Aleksandrs},
  journal={Journal of the ACM},
  volume={65},
  number={3},
  pages={1--55},
  year={2018},
  doi={10.1145/3164539}
}

@inproceedings{balseiro2020dual,
  title={Dual Mirror Descent for Online Allocation Problems},
  author={Balseiro, Santiago and Lu, Haihao and Mirrokni, Vahab},
  booktitle={Proceedings of the 37th International Conference
             on Machine Learning},
  series={Proceedings of Machine Learning Research},
  volume={119},
  pages={613--628},
  year={2020},
  publisher={PMLR}
}

@article{budget,
title = {Constrained contextual bandit algorithm for limited-budget recommendation system},
journal = {Engineering Applications of Artificial Intelligence},
volume = {128},
pages = {107558},
year = {2024},
issn = {0952-1976},
doi = {https://doi.org/10.1016/j.engappai.2023.107558},
url = {https://www.sciencedirect.com/science/article/pii/S0952197623017426},
author = {Yafei Zhao and Long Yang}
}

@Article{Ma2026EfficientSS,
 author = {Qian Ma and Ziya Chen and Zihe Wang and Lin Gao and Xu Chen},
 booktitle = {IEEE Transactions on Networking},
 journal = {IEEE Transactions on Networking},
 pages = {471-485},
 title = {Efficient Service Selection and Pricing in Edge-Cloud Computing Markets},
 volume = {34},
 year = {2026}
}

@misc{vassos2023,
      title={A simulation framework of procurement operations in the container logistics industry}, 
      author={George Vassos and Klaus K. Holst and Pierre Pinson and Richard M. Lusby},
      year={2023},
      eprint={2303.12765},
      archivePrefix={arXiv},
      primaryClass={stat.AP},
      url={https://arxiv.org/abs/2303.12765}, 
}

@inproceedings{chen2024routerdc,
  title={RouterDC: Query-Based Router by Dual Contrastive Learning for Assembling Large Language Models},
  author={Chen, Shuhao and Jiang, Weisen and Lin, Baijiong and Kwok, James T. and Zhang, Yu},
  booktitle={Advances in Neural Information Processing Systems},
  volume={37},
  year={2024}
}

@article{wei2025barp,
  title={Learning to Route LLMs from Bandit Feedback: One Policy, Many Trade-offs},
  author={Wei, Wang and Yang, Tiankai and Chen, Hongjie and Zhao, Yue and Dernoncourt, Franck and Rossi, Ryan A. and Eldardiry, Hoda},
  journal={arXiv preprint arXiv:2510.07429},
  year={2025}
}

@inproceedings{uddin-bauer-2026-conformal,
  title = {Conformal {LLM} Routing with Distribution-Free Safety Guarantees},
  author = {Uddin, Iqtedar and Bauer, Andr{\'e}},
  booktitle = {Proceedings of the 64th Annual Meeting of the
               Association for Computational Linguistics
               (Volume 4: Student Research Workshop)},
  year = {2026},
  pages = {791--799},
  publisher = {Association for Computational Linguistics},
  address = {San Diego, California, United States},
  doi = {10.18653/v1/2026.acl-srw.70},
  url = {https://aclanthology.org/2026.acl-srw.70/}
}

@misc{woisetschlaeger2026costoptimalllmrouting,
  title = {Cost-Optimal {LLM} Routing with Limited User Feedback
           under User Satisfaction Guarantees},
  author = {Woisetschl{\"a}ger, Herbert and
            Mammadli, Arastun and
            Zhang, Ryan and
            Wang, Shiqiang},
  year = {2026},
  eprint = {2606.19376},
  archivePrefix = {arXiv},
  primaryClass = {cs.LG},
  url = {https://arxiv.org/abs/2606.19376}
}

@misc{woisetschläger2026costoptimalllmroutinglimited,
  title={Cost-Optimal LLM Routing with Limited User Feedback under User Satisfaction Guarantees},
  author={Herbert Woisetschläger and Arastun Mammadli and Ryan Zhang and Shiqiang Wang},
  year={2026},
  eprint={2606.19376},
  archivePrefix={arXiv},
  primaryClass={cs.LG}
}

@misc{wei2025learning,
  title={Learning to Route LLMs from Bandit Feedback: One Policy, Many Trade-offs},
  author={Wang Wei and Tiankai Yang and Hongjie Chen and Yue Zhao and Franck Dernoncourt and Ryan A. Rossi and Hoda Eldardiry},
  year={2025},
  eprint={2510.07429},
  archivePrefix={arXiv},
  primaryClass={cs.LG}
}

@article{yang2025data,
  title={Data-Driven Minimax Optimization with Expectation Constraints},
  author={Yang, Shuoguang and Li, Xudong and Lan, Guanghui},
  journal={Operations Research},
  volume={73},
  number={3},
  pages={1345--1365},
  year={2025},
  doi={10.1287/opre.2022.0110}
}

@misc{nvidia_nemotron_3_2025,
  title  = {NVIDIA Nemotron 3: Efficient and Open Intelligence},
  author = {{NVIDIA}},
  year   = {2025},
  url    = {https://arxiv.org/abs/2512.20856},
  note   = {White Paper}
}

@misc{deepseekai2026deepseekv4,
      title={DeepSeek-V4: Towards Highly Efficient Million-Token Context Intelligence},
      author={DeepSeek-AI},
      year={2026},
}

@misc{qwen36_35b_a3b,
    title = {{Qwen3.6-35B-A3B}: Agentic Coding Power, Now Open to All},
    url = {https://qwen.ai/blog?id=qwen3.6-35b-a3b},
    author = {{Qwen Team}},
    month = {April},
    year = {2026}
}

@misc{gemmateam2026gemma4,
      title={Gemma 4 Technical Report}, 
      author={Gemma Team},
      year={2026},
      eprint={2607.02770},
      archivePrefix={arXiv},
      primaryClass={cs.CL},
      url={https://arxiv.org/abs/2607.02770}, 
}

@misc{qwen38,
    title = {{Qwen3.8-Max}: A New Bar for Coding and Cowork},
    url = {https://qwen.ai/blog?id=qwen3.8},
    author = {{Qwen Team}},
    month = {August},
    year = {2026}
}

@misc{metamodels2026museglimmer30b,
  title        = {Muse-Glimmer-30B: An Open-Weight Dense Multimodal Causal Transformer},
  author       = {Meta Superintelligence Labs},
  year         = {2026},
  howpublished = {\url{https://huggingface.co/meta-models/Muse-Glimmer-30B}},
  note         = {Apache 2.0 License}
}
\bibliographystyle{icml2025}
\appendix
\section{Proofs}

\subsection{Necessary Expectation Condition}
\label{sec:appendix_necessary_expectation}

\begin{lemma}[Necessary Expectation Condition]
Suppose the stationary contextual oracle generates i.i.d. satisfaction
outcomes under the model described in Section~\ref{sec:problem_statement}.
In particular, prompts are i.i.d., the oracle randomizes independently
across rounds according to $\pi^*(\cdot\mid p_t)$, and satisfaction
randomness is conditionally independent across rounds.
Given $\epsilon\in(0,1)$, $\delta\in(0,1)$, and $T$, define
\[
q\equiv 1-\epsilon,
\qquad
m_T\equiv \lceil Tq\rceil.
\]
Let $\mu_{T,\delta}\in(0,1)$ denote the unique value satisfying
\begin{equation}
\mathbb{P}_{\mu_{T,\delta}}
\left(
S_T\geq m_T
\right)
=
1-\delta,
\label{eq:mu_threshold}
\end{equation}
where
\[
S_T\sim\operatorname{Binomial}(T,\mu_{T,\delta}).
\]
Define the maximum expectation margin
\begin{equation}
\rho^*_{\max}(T,\epsilon,\delta)
\equiv
\mu_{T,\delta}-(1-\epsilon).
\label{eq:rho_star_max}
\end{equation}
Then, for any
\[
0\leq \rho^*
\leq
\rho^*_{\max}(T,\epsilon,\delta),
\]
satisfying the chance constraint
\begin{equation}
\mathbb{P}\left(
\frac{1}{T}\sum_{t=1}^T Y_t
\geq 1-\epsilon
\right)
\geq 1-\delta
\end{equation}
necessarily requires
\begin{equation}
\mathbb{E}\left[p_{a_t^*,t}^*\right]
\geq
1-\epsilon+\rho^*.
\label{eq:necessary_expectation}
\end{equation}
In particular,
$\rho^*_{\max}(T,\epsilon,\delta)$ is the largest expectation
margin implied by the chance constraint.
\end{lemma}

\begin{proof}
Let
\[
\mu^*
\equiv
\mathbb{E}[p_{a_t^*,t}^*]
\]
and define
\[
q\equiv 1-\epsilon,
\qquad
m_T\equiv\lceil Tq\rceil.
\]

Under the stated assumptions, the random variables
$p_{a_t^*,t}^*$ are i.i.d., and
\[
Y_t\mid p_{a_t^*,t}^*
\sim
\operatorname{Bernoulli}(p_{a_t^*,t}^*).
\]
Consequently, the marginal outcomes are i.i.d. Bernoulli with mean
$\mu^*$, so that
\[
S_T\equiv\sum_{t=1}^T Y_t
\sim
\operatorname{Binomial}(T,\mu^*).
\]

The chance constraint can therefore be written as
\[
\mathbb{P}_{\mu^*}(S_T\geq m_T)
\geq
1-\delta.
\]

Let $\mu_{T,\delta}$ be defined by
\eqref{eq:mu_threshold}. The binomial upper-tail probability
\[
\mathbb{P}_{\mu}(S_T\geq m_T)
\]
is continuous and strictly increasing in $\mu\in(0,1)$.
Therefore, if
\[
\mu^*<\mu_{T,\delta},
\]
then
\[
\mathbb{P}_{\mu^*}(S_T\geq m_T)
<
\mathbb{P}_{\mu_{T,\delta}}(S_T\geq m_T)
=
1-\delta.
\]
Thus the chance constraint cannot hold when
$\mu^*<\mu_{T,\delta}$. Hence, satisfying the chance constraint
necessarily requires
\[
\mu^*
\geq
\mu_{T,\delta}.
\]

By the definition of $\rho^*_{\max}$,
\[
\mu_{T,\delta}
=
1-\epsilon+\rho^*_{\max}(T,\epsilon,\delta).
\]
Therefore,
\[
\mathbb{E}[p_{a_t^*,t}^*]
=
\mu^*
\geq
1-\epsilon+\rho^*_{\max}(T,\epsilon,\delta).
\]
Since
\[
\rho^*
\leq
\rho^*_{\max}(T,\epsilon,\delta),
\]
we obtain
\[
\mathbb{E}[p_{a_t^*,t}^*]
\geq
1-\epsilon+\rho^*.
\]
This proves the result.
\end{proof}

\begin{corollary}[Large-Horizon Approximation]
\label{cor:rho_star_asymptotic}
For fixed $\epsilon\in(0,1)$ and $\delta\in(0,1)$, as
$T\to\infty$, the maximum expectation margin defined in
Lemma~\ref{lem:necessary_expectation} satisfies
\begin{equation}
\rho^*_{\max}(T,\epsilon,\delta)
=
\frac{
\Phi^{-1}(1-\delta)
\sqrt{\epsilon(1-\epsilon)}
}{\sqrt{T}}
+
o\left(\frac{1}{\sqrt{T}}\right),
\label{eq:rho_star_asymptotic}
\end{equation}
where $\Phi^{-1}$ denotes the inverse cumulative distribution
function of the standard normal distribution. Thus, for sufficiently
large $T$,
\begin{equation}
\boxed{
\rho^*_{\max}(T,\epsilon,\delta)
\approx
\Phi^{-1}(1-\delta)
\sqrt{\frac{\epsilon(1-\epsilon)}{T}}.
}
\end{equation}
\end{corollary}

\begin{proof}
By Lemma~\ref{lem:necessary_expectation},
$\rho^*_{\max}(T,\epsilon,\delta)$ is defined through the binomial
threshold $\mu_{T,\delta}$ satisfying
\[
\mathbb{P}_{\mu_{T,\delta}}
\left(
\frac{S_T}{T}\geq 1-\epsilon
\right)
=
1-\delta,
\]
where
$S_T\sim\operatorname{Binomial}(T,\mu_{T,\delta})$.
Write
\[
\mu_{T,\delta}
=
1-\epsilon+\rho^*_{\max}(T,\epsilon,\delta).
\]
For fixed $\epsilon$ and $\delta$, the central limit theorem gives
\[
\frac{S_T-T\mu_{T,\delta}}
{\sqrt{T\mu_{T,\delta}(1-\mu_{T,\delta})}}
\Rightarrow
N(0,1).
\]
Since $\rho^*_{\max}(T,\epsilon,\delta)\to0$ as
$T\to\infty$, we have
\[
\mu_{T,\delta}(1-\mu_{T,\delta})
\to
\epsilon(1-\epsilon).
\]
Consequently,
\[
\frac{
\rho^*_{\max}(T,\epsilon,\delta)\sqrt{T}
}{
\sqrt{\epsilon(1-\epsilon)}
}
\to
\Phi^{-1}(1-\delta),
\]
which yields
\eqref{eq:rho_star_asymptotic}.
\end{proof}

\subsection{Proof of Regret Bounds}
Before proving the regret bounds, we first list and prove the results needed for the main theorem.

\subsubsection{Proof of Lemma~\ref{lem:confidence}}
\label{sec:step1_ucb_bounds}
Define
\begin{equation}
\kappa
\equiv
\inf_{|z|\leq LS}\mu'(z)>0.
\end{equation}

For a given confidence parameter $\delta_{\mathcal E}$, define
\begin{equation*}
\alpha_t
=
\frac{1}{\kappa}
\left[
\frac{1}{2}
\sqrt{
d\log\left(
1+\frac{tL^2}{d\beta}
\right)
+
2\log\left(\frac{2}{\delta_{\mathcal E}}\right)
}
+
\sqrt{\beta}S
\right].
\label{eq:alpha_def}
\end{equation*}

Let $\mathcal E$ denote the event on which the corresponding confidence inequality holds uniformly over all rounds.

\begin{lemma}[Uniform GLM Confidence Bound]
Under Assumption~\ref{ass:boundedness}, with probability at least
$1-\delta_{\mathcal E}$, simultaneously for all $t\in[T]$ and
$a\in\mathcal A\setminus\{\tilde a\}$,
\begin{equation}
\|\widehat\theta_t-\theta^*\|_{V_t}\leq\alpha_t
\end{equation}
and
\begin{equation}
\left|
\mu(x_{a,t}^{\top}\widehat\theta_t)
-
\mu(x_{a,t}^{\top}\theta^*)
\right|
\leq
\alpha_t\|x_{a,t}\|_{V_t^{-1}}.
\end{equation}
\end{lemma}

\begin{proof}
Define the regularized negative log-likelihood at round $t$ as
\begin{equation}
\mathcal L_t(\theta)
=
\sum_{\substack{\tau=1\\a_\tau \neq \tilde{a}}}^{t-1}
l\left(
Y_\tau,
\mu(x_{a_\tau,\tau}^\top\theta)
\right)
+
\frac{\beta}{2}\|\theta\|_2^2.
\end{equation}

The estimator $\widehat\theta_t$ is obtained via constrained minimization over the ball $\mathcal B_2(S) = \{\theta \in \mathbb R^d : \|\theta\|_2 \le S\}$:
\begin{equation}
\widehat\theta_t = \arg\min_{\theta \in \mathcal B_2(S)} \mathcal L_t(\theta).
\end{equation}

Since $\mathcal L_t$ is convex and $\mathcal B_2(S)$ is a closed convex set, $\widehat\theta_t$ satisfies the first-order optimality condition:
\begin{equation}
\langle \nabla\mathcal L_t(\widehat\theta_t), \theta - \widehat\theta_t \rangle \geq 0 \quad \forall \theta \in \mathcal B_2(S).
\end{equation}

Since $\theta^* \in \mathcal B_2(S)$ by Assumption~\ref{ass:boundedness}, setting $\theta = \theta^*$ yields:
\begin{equation}
\langle \nabla\mathcal L_t(\widehat\theta_t), \widehat\theta_t - \theta^* \rangle \leq 0.
\label{eq:variational_inequality}
\end{equation}

Using the fundamental theorem of calculus along the line segment connecting $\theta^*$ and $\widehat\theta_t$, we obtain:
\begin{equation}
\nabla\mathcal L_t(\widehat\theta_t)
-
\nabla\mathcal L_t(\theta^*)
=
\bar{H}_t
(\widehat\theta_t-\theta^*),
\end{equation}
where
\begin{equation}
\bar{H}_t
=
\int_0^1
\nabla^2 \mathcal L_t\left(
\theta^*
+
s\left(\widehat{\theta}_t-\theta^*\right)
\right)
\,ds.
\end{equation}

Taking the inner product with $(\widehat\theta_t - \theta^*)$ on both sides:
\begin{align}
\|\widehat\theta_t - \theta^*\|_{\bar{H}_t}^2 
&= \langle \nabla\mathcal L_t(\widehat\theta_t) - \nabla\mathcal L_t(\theta^*), \widehat\theta_t - \theta^* \rangle \cr
&= \langle \nabla\mathcal L_t(\widehat\theta_t), \widehat\theta_t - \theta^* \rangle - \langle \nabla\mathcal L_t(\theta^*), \widehat\theta_t - \theta^* \rangle \cr
&\leq -\langle \nabla\mathcal L_t(\theta^*), \widehat\theta_t - \theta^* \rangle,
\label{eq:hessian_inner_bound}
\end{align}
where the final step applies the variational inequality \eqref{eq:variational_inequality}.

Evaluating the gradient at $\theta^*$ gives:
\begin{equation}
-\nabla\mathcal L_t(\theta^*)
=
\sum_{\substack{\tau=1\\a_\tau \neq \tilde{a}}}^{t-1}
\eta_\tau x_{a_\tau,\tau}
-
\beta\theta^*,
\end{equation}
where $\eta_\tau = Y_\tau - \mu(x_{a_\tau,\tau}^\top\theta^*)$ is a conditionally zero-mean martingale difference.

The Hessian at any point $\theta$ is:
\begin{equation}
\nabla^2\mathcal L_t(\theta)
=
\sum_{\substack{\tau=1\\a_\tau \neq \tilde{a}}}^{t-1}
\mu'(x_{a_\tau,\tau}^\top\theta)
x_{a_\tau,\tau} x_{a_\tau,\tau}^\top
+
\beta I_d.
\end{equation}

Because $\theta^*, \widehat\theta_t \in \mathcal B_2(S)$ and $\mathcal B_2(S)$ is convex, any point on the line segment $\theta(s) = \theta^* + s(\widehat\theta_t - \theta^*)$ for $s \in [0,1]$ satisfies $\|\theta(s)\|_2 \le S$. Under Assumption~\ref{ass:boundedness}, $\|x_{a_\tau,\tau}\|_2 \le L$, which guarantees:
\begin{equation}
|x_{a_\tau,\tau}^\top \theta(s)| \le L S \quad \forall s \in [0, 1].
\end{equation}

Consequently, $\mu'(x_{a_\tau,\tau}^\top \theta(s)) \ge \kappa \equiv \inf_{|z|\le LS} \mu'(z) > 0$. Therefore, using the fact that $\kappa \le 1$:
\begin{align}
\bar{H}_t 
&\succeq
\kappa \sum_{\substack{\tau=1\\a_\tau \neq \tilde{a}}}^{t-1} x_{a_\tau,\tau} x_{a_\tau,\tau}^\top + \beta I_d \cr
&\succeq
\kappa \left( \beta I_d + \sum_{\substack{\tau=1\\a_\tau \neq \tilde{a}}}^{t-1} x_{a_\tau,\tau} x_{a_\tau,\tau}^\top \right)
= \kappa V_t.
\end{align}

Applying $\bar{H}_t \succeq \kappa V_t$ to \eqref{eq:hessian_inner_bound} along with the Cauchy-Schwarz inequality in the dual norms gives:
\begin{align}
&\kappa \|\widehat\theta_t - \theta^*\|_{V_t}^2
\leq \|\widehat\theta_t - \theta^*\|_{\bar{H}_t}^2 \cr
&\leq \left\langle \sum_{\substack{\tau=1\\a_\tau \neq \tilde{a}}}^{t-1}\eta_\tau x_{a_\tau,\tau} - \beta\theta^*, \, \widehat\theta_t - \theta^* \right\rangle \cr
&\leq \left\| \sum_{\substack{\tau=1\\a_\tau \neq \tilde{a}}}^{t-1}\eta_\tau x_{a_\tau,\tau} - \beta\theta^* \right\|_{V_t^{-1}} \|\widehat\theta_t - \theta^*\|_{V_t}.
\end{align}

Dividing both sides by $\kappa \|\widehat\theta_t - \theta^*\|_{V_t}$ yields:
\begin{align*}
\|\widehat\theta_t-\theta^*\|_{V_{t}}
&\leq
\frac{1}{\kappa}
\left\|
\sum_{\substack{\tau=1\\a_\tau \neq \tilde{a}}}^{t-1}\eta_\tau x_{a_\tau,\tau}
-
\beta\theta^*
\right\|_{V_{t}^{-1}}\cr
&\leq
\frac{1}{\kappa}
\left(
\left\|
\sum_{\substack{\tau=1\\a_\tau \neq \tilde{a}}}^{t-1}\eta_\tau x_{a_\tau,\tau}
\right\|_{V_{t}^{-1}}
+
\sqrt{\beta}S
\right).
\end{align*}

Since $Y_\tau \in \{0, 1\}$, the noise $\eta_\tau$ is conditionally $1/2$-sub-Gaussian. Applying the self-normalized vector martingale inequality, with probability at least $1 - \delta_{\mathcal E}/2$, uniformly for all $t$:
\begin{align}
&\left\|
\sum_{\substack{\tau=1\\a_\tau \neq \tilde{a}}}^{t-1}\eta_\tau x_{a_\tau,\tau}
\right\|_{V_{t}^{-1}} \cr
&\leq
\frac12
\sqrt{
2\log\left(\frac{2}{\delta_{\mathcal E}}\right)
+
d\log
\left(
1+\frac{(t-1)L^2}{\beta d}
\right)
}.
\end{align}

Combining these results establishes that $\|\widehat\theta_t - \theta^*\|_{V_t} \leq \alpha_t$.

Finally, by the Mean Value Theorem on the link function $\mu(\cdot)$ and Cauchy-Schwarz:
\begin{align}
\left| \mu(x_{a,t}^\top \widehat\theta_t) - \mu(x_{a,t}^\top \theta^*) \right| 
&\leq \sup_{|z|\le LS} \mu'(z) \cdot \left| x_{a,t}^\top (\widehat\theta_t - \theta^*) \right| \cr
&\leq L_\mu \|x_{a,t}\|_{V_t^{-1}} \|\widehat\theta_t - \theta^*\|_{V_t} \cr
&\leq \alpha_t \|x_{a,t}\|_{V_t^{-1}},
\end{align}
where $L_\mu = \sup_{|z|\le LS} \mu'(z) \le 1$, completing the proof.
\end{proof}
Next, using this result, we bound the UCB terms.
\subsection{UCB Bounds}
\begin{lemma}[UCB Bounds]
On the event $\mathcal E$, for every $a\in\mathcal A$:

\begin{enumerate}
\item The prediction error satisfies
\begin{equation}
\left|
\mu(x_{a,t}^{\top}\widehat\theta_t)
-
p_{a,t}^*
\right|
\leq
\alpha_t\|x_{a,t}\|_{V_{t}^{-1}}.
\end{equation}

\item The UCB is optimistic:
\begin{equation}
    \operatorname{UCB}_t(a)
    \geq
    p_{a,t}^*.
\end{equation}

\item The optimism gap is bounded by
\begin{equation}
    \operatorname{UCB}_t(a)-p_{a,t}^*
    \leq
    2\alpha_t\|x_{a,t}\|_{V_{t}^{-1}}.
\end{equation}

\end{enumerate}
\end{lemma}
\begin{proof}
Under the confidence event,

$$
\|\widehat\theta_t-\theta^*\|_{V_{t}}
\leq\alpha_t.
$$

By the mean-value theorem and the fact that $\mu'(z)\leq1$,
\begin{align}
\left|
\mu(x_{a,t}^{\top}\widehat\theta_t)
-
\mu(x_{a,t}^{\top}\theta^*)
\right|
&\leq
\left|
x_{a,t}^{\top}
(\widehat\theta_t-\theta^*)
\right|\cr
&\leq
\|x_{a,t}\|_{V_{t}^{-1}}
\|\widehat\theta_t-\theta^*\|_{V_{t}}\cr
&\leq
\alpha_t
\|x_{a,t}\|_{V_{t}^{-1}}.
\end{align}

This proves the prediction error bound.

For optimism, the preceding inequality gives

$$
\mu(x_{a,t}^{\top}\widehat\theta_t)
+
\alpha_t\|x_{a,t}\|_{V_{t}^{-1}}
\geq
p_{a,t}^*.
$$

Since $p_{a,t}^*\leq1$, taking the minimum with $1$ preserves the inequality, and hence

$$
\operatorname{UCB}_t(a)\geq p_{a,t}^*.
$$

Finally,
\begin{align*}
\operatorname{UCB}_t(a)-p^*_{a,t}
&\leq
\left|
\mu(x_{a,t}^{\top}\widehat\theta_t)
-
p_{a,t}^*
\right|
+
\alpha_t\|x_{a,t}\|_{V_{t}^{-1}}\cr
&\leq
2\alpha_t\|x_{a,t}\|_{V_{t}^{-1}}.
\end{align*}
\end{proof}

\subsubsection{Bounding the dual variable}
\begin{lemma}[Uniform Boundedness of the Dual Variable]
Suppose Assumptions~\ref{ass:boundedness} and~\ref{ass:slater} hold.
For any $\delta\in(0,1)$, there exist
$B_\lambda=\frac{6}{(1-\xi)\epsilon}+\widetilde{\mathcal O}(1/\sqrt T)$ such that for a sufficiently large $T > T_0$,
with probability at least $1-\delta$,
\begin{equation}
\lambda_t
\leq
B_{\lambda},
\qquad
\forall t\in[T].
\end{equation}
\label{eq:dual_bound_statement}
\end{lemma}
\label{sec:step2_regret_proof}
\begin{proof}
We first establish a high-probability bound on the dual variable.
Define
\[
g_t \equiv 1-\epsilon+\rho-Y_t,
\]
and let

We condition on the pre-outcome information $\mathcal G_t$ defined in Section~\ref{sec:model}:
\[
\mu_t=\mathbb E[g_t\mid\mathcal G_t],\qquad X_t=g_t-\mu_t,\qquad \mathbb E[X_t\mid\mathcal G_t]=0.
\]
Let $\mathcal F_t:=\sigma(\mathcal G_t,Y_t)$. Since $\mathcal F_{t-1}\subseteq\mathcal G_t$, we have $\mathbb E[X_t\mid\mathcal F_{t-1}]=0$, so $(X_t,\mathcal F_t)$ is a martingale-difference sequence.
Since $Y_t \in \{0,1\}$,
\[
|X_t|\leq 1.
\]

On the confidence event, Assumption~\ref{ass:slater} implies
\[
\operatorname{UCB}_t(\tilde a)
\geq
p_{\tilde a,t}^*
= \nu
\]
By the primal minimization rule,
\[
c(a_t)-\lambda_t\operatorname{UCB}_t(a_t)
\leq
c(\tilde a)-\lambda_t\operatorname{UCB}_t(\tilde a).
\]
Consequently,
\[
\operatorname{UCB}_t(a_t)
\geq
\operatorname{UCB}_t(\tilde a)
-
\frac{C_{\max}}{\lambda_t},
\]
where
\[
C_{\max}
\equiv
\max_{a,a'\in\mathcal A}|c(a)-c(a')| \leq 1.
\]
Therefore,
\[
\operatorname{UCB}_t(a_t)
\geq
\nu
-
\frac{C_{\max}}{\lambda_t}.
\]

Using the confidence bound,
\[
p_{a_t,t}^*
\geq
\operatorname{UCB}_t(a_t)
-
2\alpha_t
\|x_{a_t,t}\|_{V_{t}^{-1}},
\]
and hence
\begin{align}
\mu_t
&=
1-\epsilon+\rho-p_{a_t,t}^*\\
&\leq
\frac{C_{\max}}{\lambda_t}
+
2\alpha_t\|x_{a_t,t}\|_{V_{t}^{-1}}
+\rho-\epsilon + (1-\nu)
\end{align}
Now, we consider the case where $\rho + (1-\nu) \leq c_1\epsilon$ where $1 > c_1 \geq \xi$. Therefore, we choose $c_1 = \xi + \rho/\epsilon$. Also, note that $\rho = \frac{C_\rho}{\sqrt{T}}$ and $\xi < 1$. Therefore, for sufficiently large $T$, $c_1 < 1.$
\begin{equation}
    \mu_t \leq
\frac{C_{\max}}{\lambda_t}
+
2\alpha_t\|x_{a_t,t}\|_{V_{t}^{-1}}
-(1-c_1)\epsilon
\end{equation}
In particular, whenever
\[
\lambda_t\geq \frac{C_{\max}}{\epsilon c_2},
\]
for $0 < c_2 < 1 - c_1$
we have
\begin{equation}
\mu_t
\leq
2\alpha_t\|x_{a_t,t}\|_{V_{t}^{-1}}
-(1-c_1-c_2)\epsilon.
\label{eq:dual_drift}
\end{equation}

Define an exceptional round to be a round satisfying
\[
2\alpha_t\|x_{a_t,t}\|_{V_{t}^{-1}}
>
(1-c_1-c_2)\epsilon
\]
On every non-exceptional round for which we have
$\lambda_t\geq C_{\max}/(c_2 \epsilon)$, equation~\eqref{eq:dual_drift} gives
\[
\mu_t\leq 0.
\]

Let
\[
\alpha_T\equiv\max_{1\leq t\leq T}\alpha_t.
\]
Every exceptional round satisfies
\[
\|x_{a_t,t}\|_{V_{t}^{-1}}
>
\frac{(1-c_1-c_2)\epsilon}{2\alpha_T}.
\]
By the elliptical potential count lemma,
the number $N_T$ of exceptional rounds satisfies
\begin{align}
N_T
&\leq
\frac{2d(2\alpha_T)^2}{(1-c_1-c_2)^2\epsilon^2}
\log\frac{\det(V_T)}{\det(V_1)}\\
&\leq
\frac{8d\alpha_T^2}{(1-c_1-c_2)^2\epsilon^2}
\log\left(
1+\frac{TL^2}{\beta d}
\right),
\label{eq:exception_count}
\end{align}
where we used $V_1=\beta I$ and
$\|x_{a_t,t}\|_2\leq L$.

We now control excursions of the dual process above the threshold
\[
a\equiv \frac{C_{\max}}{c_2\epsilon}.
\]
Suppose that $\lambda_t>a$, and let
\[
k=
\max\{s<t:\lambda_s\leq a\}
\]
be the most recent time before $t$ at which the process was at or below
$a$. Since $|g_k|\leq1$ and $\eta=C_\eta/\sqrt T$,
\[
\lambda_{k+1}
\leq
a+\eta.
\]

During the excursion, $\lambda_s>a$, and hence the conditional
drift is non-positive on all non-exceptional rounds. On an exceptional
round, we use the crude bound
\[
\mu_s\leq 1.
\]
Therefore, writing $\mathcal{I}_{ex}$ for the set of exceptional rounds,
\begin{align}
\lambda_t
&\leq
a+\eta
+
\eta\sum_{s=k+1}^{t-1}\mu_s
+
\eta\sum_{s=k+1}^{t-1}X_s\\
&\leq
a+\eta
+
\eta |\mathcal{I}_{ex}|
+
\eta\sum_{s=k+1}^{t-1}X_s\\
&\leq
a+\eta+\eta N_T
+
\eta\sum_{s=k+1}^{t-1}X_s.
\label{eq:lambda_excursion}
\end{align}

Let $S_n=\sum_{s=1}^nX_s$ and $S_0=0$. Since $(X_t)$ is a martingale-difference sequence with $|X_t|\leq1$, the maximal Azuma--Hoeffding inequality gives, with probability at least $1-\delta'$,
\[
\max_{0\leq n\leq T}|S_n|\leq\sqrt{2T\log\left(\frac{2}{\delta'}\right)}.
\]
Every interval sum is a difference of two prefix sums, so on the same event,
\[
\max_{1\leq k<t\leq T+1}\sum_{s=k}^{t-1}X_s\leq2\sqrt{2T\log\left(\frac{2}{\delta'}\right)}.
\]
Combining this inequality with
\eqref{eq:lambda_excursion}, we obtain
\[
\max_{t\leq T+1}\lambda_t
\leq
a+\eta+\eta N_T
+
2\eta\sqrt{2T\log\left(\frac{2}{\delta'}\right)}
\]
with probability at least $1-\delta'-\delta_{\mathcal E}$.

Using $\eta=C_\eta/\sqrt T$ and the bound in
\eqref{eq:exception_count}, it follows that
\begin{align*}
\max_{t\leq T+1}\lambda_t
&\leq
\frac{1}{c_2\epsilon}
+
\frac{C_\eta}{\sqrt T}
+
2C_\eta\sqrt{2\log\left(\frac{2}{\delta^\prime}\right)}
 \cr
& +
\frac{8d C_\eta\alpha_T^2}{(1-c_1-c_2)^2\epsilon^2\sqrt T}
\log\left(
1+\frac{TL^2}{\beta d}
\right)\cr
&\equiv B_{\lambda}(T,\delta)
\end{align*}
with probability at least $1-\delta'-\delta_{\mathcal E}$.

We choose $\delta_{\mathcal E}=T^{-\zeta}$ and
$\delta'=\delta-\delta_{\mathcal E}$, with $\zeta>2$.

\
Choose instead $c_2=(1-c_1)/2$. Then $1-c_1-c_2=(1-c_1)/2$ stays bounded away from zero for fixed $\xi<1$, so the exceptional-round term is $\widetilde{\mathcal O}(1/\sqrt T)$. Therefore,
\begin{equation}
    B_{\lambda} = \frac{2}{(1-c_1)\epsilon} + 2C_\eta\sqrt{2\log\left(\frac{2}{\delta^\prime}\right)} + \widetilde{\mathcal O}(1/\sqrt{T}).
\end{equation}
For sufficiently large $T$, $c_1=\xi+\rho/\epsilon=\xi+\mathcal O(T^{-1/2})$ and $\delta'=\delta-T^{-\zeta}\to\delta$. Now, for a sufficiently large $T > T_0$, $\frac{1}{1-c_1} \leq \frac{2}{1-\xi}$ and $\delta' \geq \delta/2$. Substituting $C_\eta=1/((1-\xi)\epsilon\sqrt{4\log(2/\delta)})$ gives
\begin{equation}
    B_{\lambda} = \frac{6}{(1-\xi)\epsilon} + \widetilde{\mathcal O}(1/\sqrt{T}).
\end{equation}




\end{proof}

\subsection{Proof of the Cost Regret Bound}
\label{sec:regret_proof_detailed}

\begin{proof}[Proof of Theorem~\ref{thm:cost_regret}, regret bound]

We first define $\mathcal F_t \equiv \sigma(p_1,a_1,Y_1,\dots,p_{t-1},a_{t-1},Y_{t-1},p_t)$, so that $a_t$ (hence $x_{a_t,t}$ and $\lambda_t$) is $\mathcal F_t$-measurable, while $Y_t$ and $a_t^*\sim\pi^*(\cdot\mid p_t)$ (drawn with independent external randomness) are not.
Let also $\mathcal H_{t-1}\equiv\sigma(p_1,a_1,Y_1,\dots,p_{t-1},a_{t-1},Y_{t-1})$ denote the history up to the end of round $t-1$. Then $\widehat\theta_t$, $V_t$, $\alpha_t$ and $\lambda_t$ are $\mathcal H_{t-1}$-measurable, while $p_t$ and $a_t^*$ are independent of $\mathcal H_{t-1}$.

Fix the failure level $\delta'=1/T$, used only in this proof, and let $B_{\lambda}=B_{\lambda}(T,1/T)$ be the bound of Lemma~\ref{lem:dual_bound} on $\max_{1\le t\le T+1}\lambda_t$ at this level.

Define
\begin{align}
\tau \;\equiv\; &\min\Big\{\, t\in\{1,\dots,T+1\} \;:\; \mathcal E \text{ fails at round } t,\ \cr &\text{ or }\ \lambda_t > B_{\lambda} \,\Big\},
\label{eq:tau_def}
\end{align}
with $\tau=T+1$ if neither event occurs. Therefore,
\begin{equation}
\mathbb P(\tau\le T) \;\le\; \delta_{\mathcal E} + \delta' \;=\; \delta_{\mathcal E} + \frac1T.
\label{eq:tau_prob}
\end{equation}
By definition of $\tau$ as the \emph{first} crossing time, $\lambda_t\le B_{\lambda}$ holds only for $t<\tau$, not at $t=\tau$ itself; we therefore work with $T'\equiv(\tau-1)\wedge T$ throughout, rather than $\tau\wedge T$.

\textbf{Step 1: Primal comparison.}
By the selection rule of Algorithm~\ref{alg:primal_dual_glm_ucb_final}, for every $t$,
\[
c(a_t)-\lambda_t\,\mathrm{UCB}_t(a_t) \;\le\; c(a_t^*)-\lambda_t\,\mathrm{UCB}_t(a_t^*),
\]
so
\begin{equation}
c(a_t)-c(a_t^*) \;\le\; \lambda_t\Big(\mathrm{UCB}_t(a_t)-\mathrm{UCB}_t(a_t^*)\Big).
\label{eq:step1}
\end{equation}

\textbf{Step 2: Replacing UCB by the true satisfaction probabilities.}
By Lemma~\ref{lem:ucb_bounds}, on $\mathcal E$, for every $a\in\mathcal A$,
\[
p_{a,t}^* \;\le\; \mathrm{UCB}_t(a) \;\le\; p_{a,t}^* + 2\alpha_t\|x_{a,t}\|_{V_t^{-1}}.
\]
For $t\le T'$, $\mathcal E$ holds at round $t$ by definition of $\tau$, so
\begin{equation}
\mathrm{UCB}_t(a_t)-\mathrm{UCB}_t(a_t^*)
\le
\big(p^*_{a_t,t}-p^*_{a_t^*,t}\big) + 2\alpha_t\|x_{a_t,t}\|_{V_t^{-1}}.
\label{eq:step2}
\end{equation}
Combining \eqref{eq:step1}--\eqref{eq:step2}, for $t\le T'$,
\begin{equation}
c(a_t)-c(a_t^*)
\;\le\;
\lambda_t\big(p^*_{a_t,t}-p^*_{a_t^*,t}\big) + 2\lambda_t\alpha_t\|x_{a_t,t}\|_{V_t^{-1}}.
\label{eq:step2final}
\end{equation}

\paragraph{Step 3: Decomposition.}
With $g_t\equiv1-\epsilon+\rho-Y_t$, $\mu_t\equiv\mathbb E[g_t\mid\mathcal F_t]=1-\epsilon+\rho-p^*_{a_t,t}$, $X_t\equiv g_t-\mu_t=p^*_{a_t,t}-Y_t$, and $\nu_t^*\equiv1-\epsilon+\rho-p^*_{a_t^*,t}$, adding and subtracting $1-\epsilon+\rho$ in \eqref{eq:step2final} gives, for $t\le T'$,
\begin{equation}
c(a_t)-c(a_t^*) \;\le\; \underbrace{-\lambda_t\mu_t}_{(A)_t} \;+\; \underbrace{\lambda_t\nu_t^*}_{(B)_t} \;+\; \underbrace{2\lambda_t\alpha_t\|x_{a_t,t}\|_{V_t^{-1}}}_{(C)_t}.
\label{eq:decomposition}
\end{equation}
We bound $\sum_{t=1}^{T'}\mathbb E[(A)_t]$, $\sum_{t=1}^{T'}\mathbb E[(B)_t]$, and $\sum_{t=1}^{T'}(C)_t$ in Steps 4--6.

\paragraph{Step 4: Term $(A)$.}
From $\lambda_{t+1}=\max\{0,\lambda_t+\eta g_t\}$ and $|g_t|\le1$,
\begin{equation}
-\lambda_t g_t \;\le\; \frac{\lambda_t^2-\lambda_{t+1}^2}{2\eta} + \frac{\eta}{2}
\qquad\text{(pathwise)}.
\label{eq:pathwise_ineq}
\end{equation}
Since $(A)_t=-\lambda_t g_t+\lambda_t X_t$, summing \eqref{eq:pathwise_ineq} over $t=1,\dots,T'$ and using $\lambda_1=0$,
\begin{equation}
\sum_{t=1}^{T'}(A)_t \;\le\; \frac{\eta T}{2} + \sum_{t=1}^{T'}\lambda_t X_t.
\label{eq:termA_pathwise}
\end{equation}
Since $\lambda_t$ is $\mathcal F_t$-measurable and $|\lambda_t X_t|\le B_{\lambda}$ for $t\le T'$, $\big(\sum_{t=1}^{n\wedge(\tau-1)}\lambda_t X_t\big)_{n\ge0}$ is a bounded martingale, so by optional stopping,
\begin{equation}
\mathbb E\Big[\sum_{t=1}^{T'}\lambda_t X_t\Big] = 0
\quad\Longrightarrow\quad
\mathbb E\Big[\sum_{t=1}^{T'}(A)_t\Big] \;\le\; \frac{\eta T}{2}.
\label{eq:termA_final}
\end{equation}

\paragraph{Step 5: Term $(B)$.}
By the proof of Lemma~\ref{lem:necessary_expectation}, feasibility of $\pi^*$ implies $\mu^*\equiv\mathbb E[p^*_{a_t^*,t}]\ge \mu_{T,\delta}=1-\epsilon+\rho^*_{\max}$. Because $p_t$ and $a_t^*$ are independent of $\mathcal H_{t-1}$ (prompts are i.i.d.\ and $\pi^*$ is stationary with external randomization), $\mathbb E[p^*_{a_t^*,t}\mid\mathcal H_{t-1}]=\mu^*$. Hence
\begin{align}
\mathbb E\big[1-\epsilon+\rho-p^*_{a_t^*,t}\,\big|\,\mathcal H_{t-1}\big] &=\rho-(\mu^*-(1-\epsilon)) \cr
&\le \rho-\rho^*_{\max}\cr
&\le (\rho-\rho^*_{\max})^+ .
\end{align}
Since $\mathbf 1\{t\le T'\}\lambda_t$ is $\mathcal H_{t-1}$-measurable with $0\le\lambda_t\le B_{\lambda}$ on $\{t\le T'\}$,
\begin{align}
\label{eq:termB_final}
&\mathbb E\Big[\sum_{t=1}^{T'}(B)_t\Big] \cr
&=
\sum_{t=1}^T
\mathbb{E}\!\left[
\mathbf{1}\{t\leq T'\}\lambda_t\,
\mathbb{E}\!\left[
1-\epsilon+\rho-p^*_{a_t^*,t}
\mid\mathcal{H}_{t-1}
\right]
\right] \nonumber\\
&\leq T\,B_{\lambda}\,(\rho-\rho^*_{\max})^+ 
\end{align}
(If $\rho\le\rho^*_{\max}$ this is $0$)

\paragraph{Step 6: Term $(C)$.}
For $t\le T'$, $\lambda_t\le B_{\lambda}$ and $\alpha_t\le\alpha_T$. By Cauchy--Schwarz and the elliptical potential lemma (with $\beta\ge L^2$),
\begin{equation}
\sum_{t=1}^{T'}(C)_t
\;\le\;
2B_{\lambda}\alpha_T\sqrt{2dT\log\Big(1+\frac{TL^2}{\beta d}\Big)}.
\label{eq:termC_final}
\end{equation}

\paragraph{Step 7: Cost of the bad event.}
Rounds $T'+1,\dots,T$ (nonempty only if $\tau\le T$, i.e.\ the stopping round $\tau$ itself plus everything after) each contribute at most $1$, so
\begin{align}
\mathbb E\Big[\sum_{t=T'+1}^{T}\big(c(a_t)-c(a_t^*)\big)\Big]
&\;\le\;
T\cdot\mathbb P(\tau\le T)\cr
&\;\le\;
T\delta_{\mathcal E}+1 \cr
&\;\le\; 2
\label{eq:termD_final}
\end{align}
using $\delta_{\mathcal E}=T^{-\zeta}$, $\zeta\ge1$.

\paragraph{Step 8: Combining.}
Summing \eqref{eq:decomposition} and splitting at $T'$, then applying \eqref{eq:termA_final}, \eqref{eq:termB_final}, \eqref{eq:termC_final}, \eqref{eq:termD_final},
\begin{align}
\mathrm{Regret}_T
&\;\le\;
\frac{\eta T}{2} + 2B_{\lambda}\alpha_T\sqrt{2dT\log\Big(1+\frac{TL^2}{\beta d}\Big)} \nonumber\\
&+ T B_{\lambda}(\rho-\rho^*_{\max})^+ + 2.
\label{eq:combined}
\end{align}
With $\eta=C_\eta/\sqrt T$ and, from the crude (Azuma-based) form of Lemma~\ref{lem:dual_bound} at $\delta'=1/T$,

\[
B_{\lambda} \;=\; \mathcal{O}(1/\big(\epsilon(1-\xi)\big)),
\qquad
\alpha_T \;=\;\mathcal O\Big(\frac{\sqrt{d \log T}}{\kappa}\Big)
\]
(the $\log(1/\delta')=\log T$ dependence is absorbed into $\widetilde{\mathcal O}(\cdot)$).
Choose $\rho=C_\rho/\sqrt T$. Then $(\rho-\rho^*_{\max})^+\le\rho$ and
\[
T B_{\lambda}(\rho-\rho^*_{\max})^+\le C_\rho B_\lambda\sqrt T=\widetilde{\mathcal O}\!\Big(\frac{\sqrt{\log(4/\delta)}\sqrt T}{(1-\xi) \epsilon}\Big).
\]
Therefore \eqref{eq:combined} gives
\begin{align*}
&\mathrm{Regret}_T\cr
&\le
\frac{C_\eta}{2}\sqrt T
+
\widetilde{\mathcal O}\!\left(\frac{d\sqrt T}{\kappa(1-\xi)\epsilon}\right) \cr
&+\widetilde{\mathcal O}\!\left(\frac{\sqrt{\log (4/\delta)}\sqrt T}{(1-\xi)\epsilon}\right)
+
2 \cr
&=\widetilde{\mathcal O}\!\left(
\frac{(d+\sqrt{\log(4/\delta)} + \sqrt{1/\log(4/\delta)})\sqrt{T}}
{\kappa(1-\xi)\epsilon}
\right) \cr
&= \widetilde{\mathcal O}\!\left(
\frac{(d+\sqrt{\log(4/\delta)})\sqrt{T}}
{\kappa(1-\xi)\epsilon}
\right)
\end{align*}
We can show that for $\delta \leq 1$, $C_\eta\sqrt T/2 \leq C_\rho B_\lambda\sqrt{T}$, therefore, we can simplify the expression to get the last term.
\end{proof}
\subsection{Proof of High-Probability Quality Constraint Satisfaction}
\label{sec:quality_bound}
\begin{proof}
\label{sec:step3_sla_proof}

Consider the dual update

$$
\lambda_{t+1}
=
\max\left\{
0,\,
\lambda_t
+
\eta
\left(
1-\epsilon+\rho-Y_t
\right)
\right\}.
$$

Since $\max\{0,x\}\ge x$,
$$
\lambda_{t+1}
\geq
\lambda_t
+
\eta(1-\epsilon)
+
\eta\rho
-
\eta Y_t.
$$

Summing this inequality from $t=1$ to $T$ and using $\lambda_1=0$ gives
\begin{align}
\lambda_{T+1}
&\geq
\eta
\sum_{t=1}^{T}
\left(
1-\epsilon+\rho-Y_t
\right)\cr
&=
\eta(1-\epsilon+\rho)T
-
\eta\sum_{t=1}^{T}Y_t.
\end{align}

Rearranging,
\begin{equation}
\sum_{t=1}^{T}Y_t
\geq
(1-\epsilon+\rho)T
-
\frac{\lambda_{T+1}}{\eta}.
\label{eq:satisfaction_telescoped}
\end{equation}

On the event where the dual variable is uniformly bounded,

$$
\lambda_{T+1}\leq B_\lambda.
$$

Therefore,
\begin{equation}
\sum_{t=1}^{T}Y_t
\geq
(1-\epsilon)T
+
\rho T
-
\frac{B_\lambda}{\eta}.
\label{eq:satisfaction_telescoped_2}
\end{equation}

Finally, we have

$$
\eta=\frac{C_\eta}{\sqrt T},
\qquad
\rho = \frac{C_\rho}{\sqrt T}.
$$

Choose
\[
C_\rho>\frac{B_\lambda}{C_\eta}.
\]
Then
\[
\rho T-\frac{B_\lambda}{\eta}
=
\left(C_\rho-\frac{B_\lambda}{C_\eta}\right)
\sqrt T
>0,
\]
which is sufficient for the satisfaction guarantee.

Using the queue bound above with $\delta'=\delta-\delta_{\mathcal E}$ and $\delta_{\mathcal E}=T^{-\zeta}$, its leading terms satisfy
\[
\frac{B_\lambda}{C_\eta}
\leq
\frac{1}{(1-\xi)\epsilon C_\eta}
+\sqrt{2\log\!\left(\frac{4}{\delta}\right)}
+o(1).
\]
For $C_\eta=((1-\xi)\epsilon\sqrt{2\log(4/\delta)})^{-1}$, the right-hand side converges to
$2\sqrt{2\log(4/\delta)}$ as $T\to\infty$. Since
$C_\rho=5\sqrt{2\log(4/\delta)}$, the strict inequality
$C_\rho>B_\lambda/C_\eta$ therefore holds for all sufficiently large $T$ under the stated queue bound. Hence there exists $T_0$ such that, for $T\geq T_0$, the chance constraint holds with probability at least $1-\delta$.

Consequently, 

$$
\sum_{t=1}^{T}Y_t
\geq
(1-\epsilon)T.
$$

Since this conclusion holds on an event of probability at least $1-\delta$, we obtain

$$
\mathbb P
\left(
\frac{1}{T}
\sum_{t=1}^{T}Y_t
\geq
1-\epsilon
\right)
\geq
1-\delta.
$$

This establishes that our algorithm satisfies the quality constraint.
\end{proof}

\begin{figure}
    \centering
    \includegraphics[
        width=0.48\textwidth,
    ]{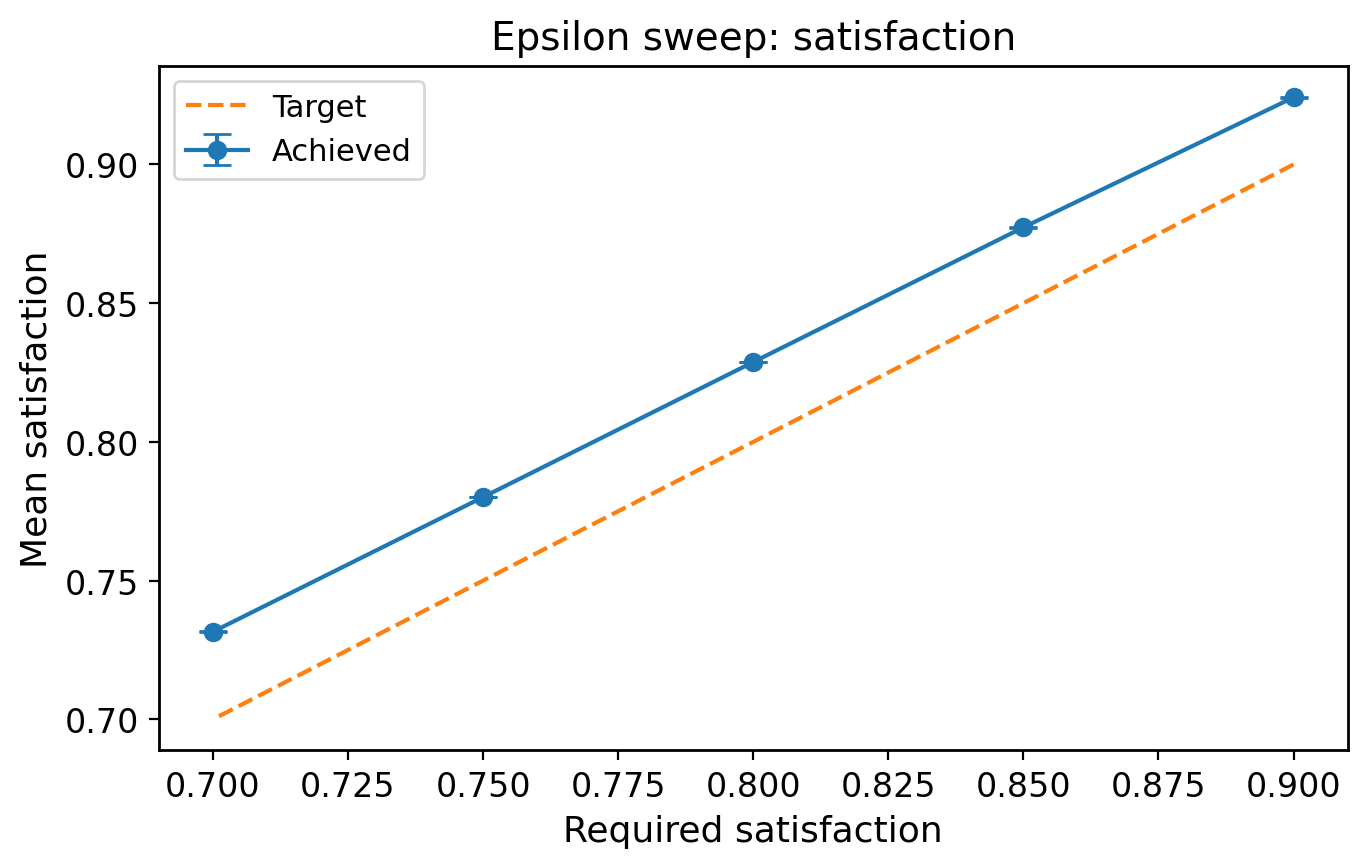}
    \caption{Sensitivity to the satisfaction tolerance $\epsilon$. Dashed lines indicate the corresponding satisfaction
    and constraint thresholds.}
    \label{fig:epsilon_sweep}
\end{figure}

\begin{figure}
    \centering
    \includegraphics[width=0.95\linewidth]{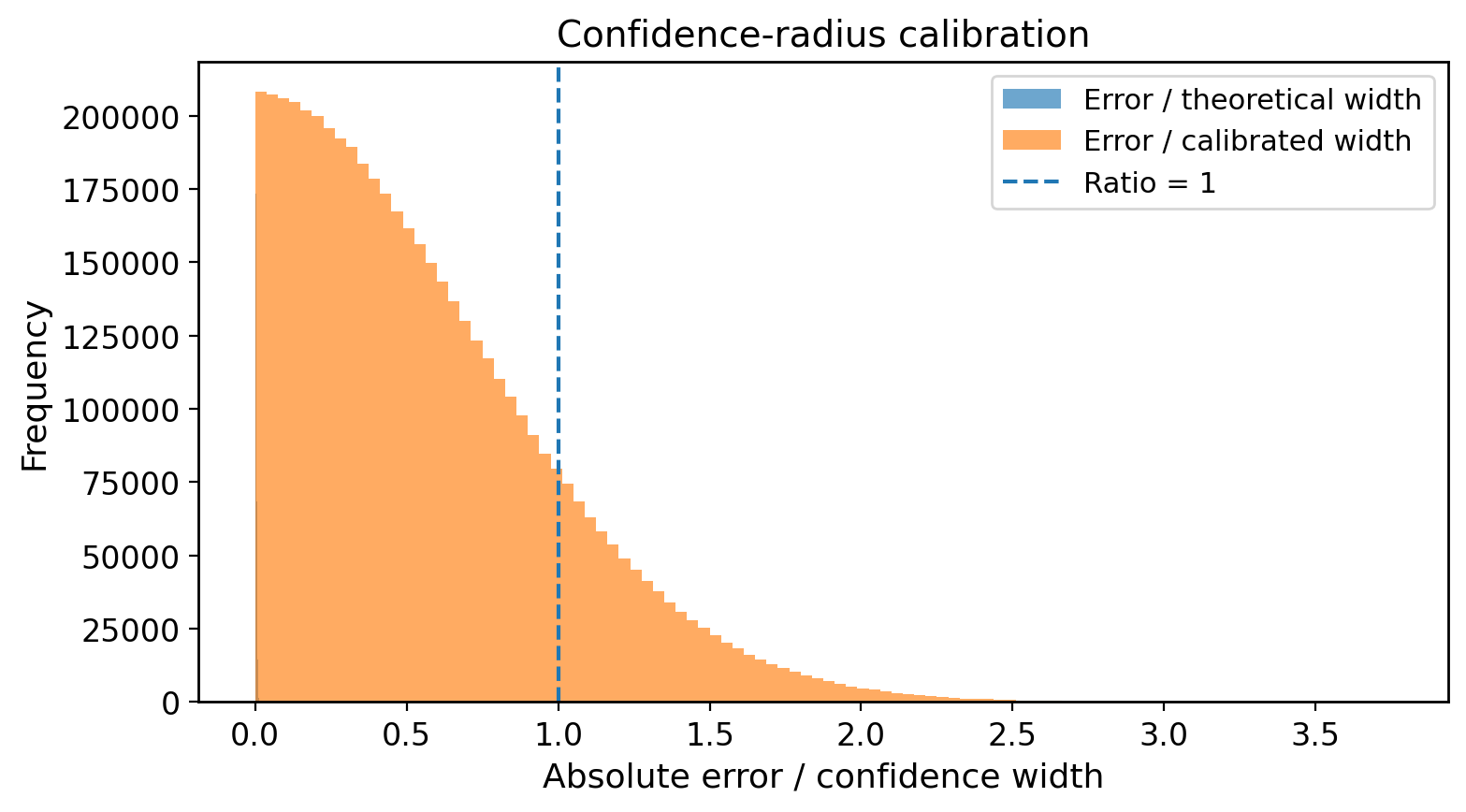}
    \caption{Confidence Calibration}
    \label{fig:confidence_calibration}
\end{figure}

\begin{figure}
    \centering
    \includegraphics[width=0.9\linewidth]{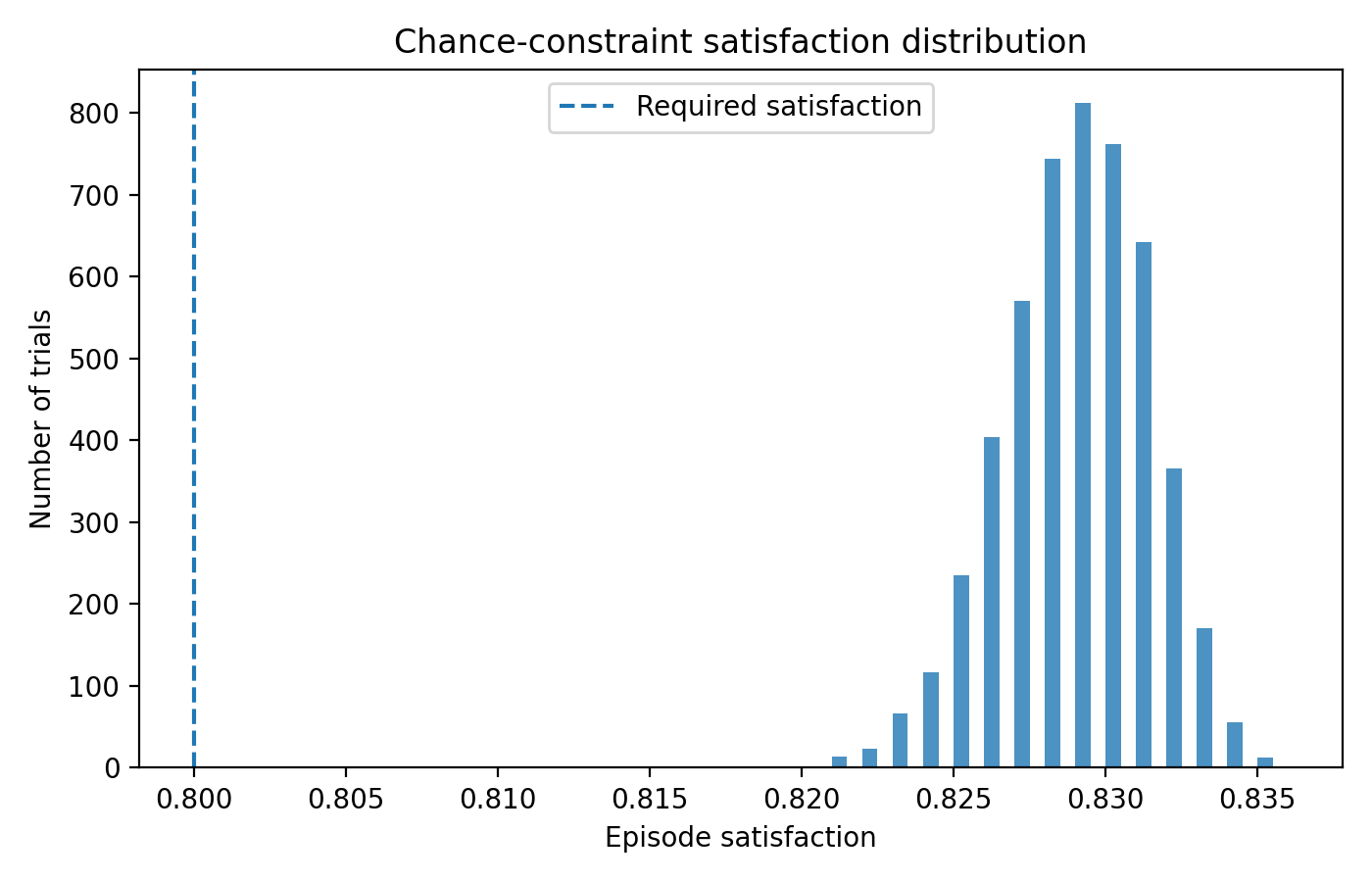}
    \caption{Empirical probability of satisfying the chance constraint over 5000 independent trials.}
    \label{fig:chance_constraint}
\end{figure}
    
\section{Additional Experiments}
\label{section:expts_add}
\subsection{Experimental Details}

All experiments are implemented in Python. For the synthetic experiments, we consider feature dimensions $d=8$ and $K=5$ routing actions. Unless otherwise stated, we use $T=1000$ rounds, $\epsilon=0.2$, $\delta=0.05$, and initialize the model parameters and covariance matrix as specified in the algorithm. Results are averaged over multiple independent runs, with confidence intervals computed across runs.

For our experiments, for $T=100$, we fix $\eta = 0.5$, $\rho = 0.15$ for $T=100$ and further scale them with respect to $1/\sqrt{T}$ for other values of $T$.
Similarly, in the experiments we multiply the theoretical confidence radius $\alpha_t$ by $0.005$ as an empirical calibration. This calibrated radius is used only for the experiments. We report their constraint-satisfaction performance empirically.

All experiments were conducted on a CPU-based machine equipped with an Intel Core i5-14600KF processor and 32\, GB of RAM. The total runtime for all experiments was approximately six hours.

For the real-LLM experiments, we use 100 queries from WildBench v2 and five LLMs as the available routing actions. The LLMs are mentioned as follows:
\begin{enumerate}
    \item nemotron-3-super-120b-a12b \cite{nvidia_nemotron_3_2025}
    \item deepseek-v4-flash \cite{deepseekai2026deepseekv4}
    \item qwen3.6-35b-a3b \cite{qwen36_35b_a3b}
    \item gemma-4-31b-it \cite{gemmateam2026gemma4}
    \item qwen3.8-27b \cite{qwen38}
\end{enumerate}
Each response is evaluated by an independent lightweight judge model (muse-glimmer-30b \cite{metamodels2026museglimmer30b}) using six quality dimensions: correctness, relevance, completeness, clarity, conciseness, and safety. The resulting quality scores are used to compute the satisfaction outcome and routing cost. All methods use the same queries, models, costs, and evaluation procedure to ensure a fair comparison. The code for all the experiments can be found at \blue{https://github.com/ameyanjarlekar/LLM-Routing}, and the models are hosted at \blue{https://lumen.ncsa.illinois.edu/}.

\subsection{Sensitivity to the Satisfaction Tolerance}
\label{sec:exp_epsilon}

We study how the routing performance changes with the satisfaction tolerance $\epsilon$. A smaller $\epsilon$ imposes a stricter satisfaction requirement, while a larger $\epsilon$ allows greater flexibility in model selection. We vary $\epsilon$ over a range of values and evaluate the resulting routing cost, satisfaction, and chance-constraint probability.

\subsection{Confidence Width Calibration}
\label{sec:exp_calibration}

We evaluate the tightness of the theoretical confidence width for the predicted satisfaction probabilities. At $T=1000$, the raw theoretical width is highly conservative in our evaluated data. We also evaluate a calibrated width obtained by applying a fixed scaling factor of $0.005$. The calibrated width covers $83.4\%$ of the observed errors in this experiment, which justifies our choice of multiplying the theoretical confidence radius $\alpha_t$ by $0.005.$

\subsection{Quality-Constraint Satisfaction}

We next evaluate the reliability of constraint satisfaction across independent trials. For each independent trial, we run the algorithm for $T=1000$ rounds and record whether the realized average satisfaction meets the target $1-\epsilon=0.8$. We repeat this procedure over 5000 independent trials and report the empirical constraint-satisfaction frequency with $95\%$ confidence intervals. Figure~\ref{fig:chance_constraint} compares this empirical frequency with the target $1-\delta=0.95$
\end{document}